\documentclass{article}
\usepackage{matus}

\def\nL{\nabla L}
\def\nR{\nabla\cR}
\def\AT{A^\top}
\def\Pip{\Pi_\perp}
\def\llog{\ell_{\log}}
\def\lexp{\ell_{\exp}}
\def\cRlog{\cR_{\log}}

\def\heta{\hat\eta}

\newenvironment{proofof}[1]{\begin{proof}\textbf{(of {#1})}}{\end{proof}}

\title{Risk and parameter convergence of logistic regression}
\iftrue
\author{Ziwei Ji\qquad Matus Telgarsky\\
\tt{\{ziweiji2,mjt\}@illinois.edu}\\
University of Illinois, Urbana-Champaign}
\else
\author{Ziwei Ji\thanks{\tt{<\url{ziweiji2@illinois.edu}>} and \tt{<\url{mjt@illinois.edu}>};
University of Illinois, Urbana-Champaign.}
\and
Matus Telgarsky\footnotemark[1]}
\fi
\date{}

\begin{document}

\maketitle

\begin{abstract}
  Gradient descent, when applied to the task of logistic regression,
  outputs iterates which are biased to follow a unique ray defined by the data.
  The direction of this ray is the maximum margin predictor of a maximal linearly separable
  subset of the data; the gradient descent iterates converge to this ray \emph{in direction}
  at the rate $\cO(\nicefrac{\ln\ln t }{\ln t})$.
  The ray does not pass through the origin in general, and its offset is the bounded
  global optimum of the risk over the remaining data;
 gradient descent recovers this offset at a rate $\cO(\nicefrac{(\ln t)^2}{\sqrt{t}})$.
\end{abstract}

\section{Introduction}

Logistic regression is the task of finding a vector $w\in\R^d$ which approximately minimizes
the \emph{empirical logistic risk}, namely
\[
  \cRlog(w) := \frac 1 n \sum_{i=1}^n \llog(\ip{w}{-x_iy_i})
  \qquad
  \textup{with}
  \quad
  \llog(r) := \ln\del{1 + \exp(r)},
\]
where $\llog$ is the \emph{logistic loss}.  A traditional way to minimize $\cRlog$ is to pick
an arbitrary $w_0$, and from there recursively construct \emph{gradient descent iterates}
$(w_j)_{j\geq 0}$
via $w_{j+1} := w_j - \eta_j \nabla \cRlog(w_j)$,
where $(\eta_j)_{j\geq 0}$ are step size parameters.

Despite the simplicity of this setting, a general characterization of the gradient descent
path has escaped the literature for a variety of reasons.
It is possible for the data to be configured so that $\cRlog$ is strongly convex,
in which case standard convex optimization tools grant the existence of a unique bounded
optimum, and moreover a rate at which gradient descent iterates converge to it.
It is also possible, however, that data is \emph{linearly separable}, in which case
$\cRlog$ has an infimum of 0 despite being positive everywhere; the optimum is off at infinity,
and convergence analyses operate by establishing a \emph{maximum margin} property of the
normalized iterates $\nicefrac{w_j}{|w_j|}$ \citep{nati_logistic},
just as in the analysis of AdaBoost \citep{schapire_freund_book_final,mjt_margins}.

In general, data can fail to induce a strongly convex risk or
a linearly separable problem.
Despite this, there is still a unique characterization of the gradient descent path.
Specifically, gradient descent is biased to follow an \emph{optimal ray}
$\cbr{ \barv + r\cdot\baru\ : \ r\geq 0}$, which is constructed as follows.
First, as detailed in \Cref{sec:struct}, the data uniquely determines (via a greedy procedure)
a linearly separable subset, with a corresponding maximum margin predictor $\baru$;
gradient descent converges to $\baru$ \emph{in direction}, meaning $\nicefrac {w_j}{|w_j|} \to \baru$.
The remaining data span a space $S$; the empirical risk of the remaining data
is strongly convex over bounded subsets of $S$, and possesses a unique optimum $\barv$,
to which the projected gradient descent iterates converge, meaning $\Pi_S w_j \to \barv$.

\begin{theorem}[name={Simplification of \Cref{fact:struct,thm:risk_converge,fact:min_norm}}]
  \label{fact:main}
  Let examples $((x_i,y_i))_{i=1}^n$ be given satisfying $|x_iy_i|\leq 1$,
  along with a loss $\ell \in\{\llog, \exp\}$, with corresponding risk $\cR$ as above.
  Consider gradient descent iterates $(w_j)_{j\geq 0}$ as above, with $w_0 = 0$.
  \begin{enumerate}
    \item\textbf{(Convergence in risk.)}
      For any step sizes $\eta_j \leq 1$ and any $t\geq 1$,
      \[
        \cR(w_t) - \inf_{w\in\R^d} \cR(w)
        =
        \cO\del{
          \frac 1 t + \frac{\ln(t)^2}{\sum_{j<t}\eta_j}
        }
        =
        \begin{cases}
          \cO(\ln(t)^2/t)
          &\eta_j = \Omega(1),
          \\
          \cO(\ln(t)^2/\sqrt{t})
          &\eta_j = \Omega(1/\sqrt{j+1}),
        \end{cases}
      \]
      where $\cO(\cdot)$ hides problem-dependent constants.

    \item\textbf{(Convergence in parameters; implicit bias and regularization.)}
      The data uniquely determines a subspace $S$
      and a vector $\barv \in S$,
      such that if $\eta_j := 1/\sqrt{j+1}$
      and $t^2 = \Omega(n\ln(t))$,
      letting $\Pi_S$ denote orthogonal projection onto $S$
      and $\barw_t := \argmin\cbr{ \cR(w) : |w| \leq |w_t|}$
      denote the solution to the constrained optimization problem,
      then
      \begin{align*}
        |\Pi_S w_t| = \Theta(1)
        \qquad&\textup{and}\qquad
        \max\cbr{ |\Pi_S w_t - \barv|^2,\ {} |\Pi_S \barw_t - \barv|^2 } = \cO\del{\frac{\ln(t)^2}{\sqrt{t}}}.
      \end{align*}
      If there are examples outside $S$, their projection onto $S^\perp$ is linearly separable with maximum margin
      predictor $\baru\in S^\perp$,
      and
      \begin{align*}
          |\Pi_{S^\perp} w_t| = \Theta(\ln(t))
          \qquad&\textup{and}\qquad
          \max\cbr{ \envert{\frac{w_t}{|w_t|} - \baru }^2,\ {}
                    \envert{\frac{\barw_t}{|\barw_t|} - \baru }^2 }
          = \cO\del{\frac{\ln\ln t}{\ln t}}.
      \end{align*}
      In particular, $\nicefrac{w_t}{|w_t|} \to \baru$ and  $\Pi_S w_t \to \barv$.
  \end{enumerate}
\end{theorem}

This theorem captures \emph{implicit bias} by showing that gradient descent follows the unique
ray $\{ \barv + r\cdot\baru : r\geq 0\}$, even though the risk itself may be minimized by any
vector which lies in the relative interior of a convex cone defined by the problem.
Similarly, the theorem captures \emph{implicit regularization} by showing that the gradient
descent iterates also track the sequence of constrained optima $(\bar w_j)_{j\geq 1}$.

This paper is organized as follows.

  \paragraph{Problem structure (\Cref{sec:struct}).}
    This section first builds up the case of general data with a few illustrative examples,
    including strongly convex and linearly separable cases.
    Thereafter, a complete construction and characterization of the optimal ray
    $\{\barv + r\cdot\baru : r\geq 0\}$ and related objects is provided in \Cref{fact:struct}.

  \paragraph{Risk convergence (\Cref{sec:risk}).}
    The preceding section on problem structure reveals that the (boun\-ded) point
    $\barv + \baru\del{\nicefrac{\ln(t)}{\gamma}}$ achieves low risk;
    plugging this into a modified smoothness argument yields converge in risk with no
    apparent dependence on the optimum at infinity.

  \paragraph{Parameter convergence (\Cref{sec:param}).}
  The preceding problem structure reveals $\cR$ is strongly convex over bounded subsets of
  $S$, which gives convergence to $\barv$ via standard convex optimization tools.

    To prove $\nicefrac{w_j}{|w_j|} \to \baru$,
    the first key to the analysis is to study not $\cR$ but instead $\ln\cR$,
    which more conveniently captures
    local smoothness (extreme flattening) of $\cR$.
    To complete the proof, a number of technical issues
    must be worked out, including bounds on $|w_t|$, which rely upon an adaptation of the
    perceptron convergence proof.
       This proof goes through much more easily for the exponential loss,
     which is the main reason for its appearance in \Cref{fact:main}.

  \paragraph{Related work (\Cref{sec:related}).}
    The paper closes with a discussion of related work.

\subsection{Notation}

The data sample will be denoted by $((x_i,y_i))_{i=1}^n$.
Collect these examples into a matrix $A \in \R^{n\times d}$, with $i^{\textup{th}}$ row $A_i := -y_ix_i^\top$;
it is assumed that $\max_i |A_i| = \max_i |x_iy_i| \leq 1$, where $|\cdot|$ denote the $\ell_2$-norm in this paper.
Given loss function $\ell:\R\to\R_{\geq 0}$, for any $k$ and any $v\in \mathbb{R}^k$, define a coordinate-wise form
$L(v) := \sum_{i=1}^k \ell(v_i)$, whereby the empirical risk
$\cR(w) := \sum_{i=1}^n \ell\left(\ip{w}{-x_iy_i}\right)/n$
satisfies $\cR(w) = L(Aw)/n$, with gradient $\nR(w) = \AT\nL(Aw)/n$.
Define $\lexp(z) := \exp(z)$ and $\llog(z) := \ln(1+\exp(z))$,
and correspondingly $L_{\exp}$, $\cR_{\exp}$, $L_{\log}$, and $\cR_{\log}$.

As in \Cref{fact:main}, and will be elaborated in \Cref{sec:struct},
the matrix $A$ defines a unique division of $\R^d$
into a direct sum of subspaces $\R^d = S \oplus S^\perp$.
The rows of $A$ are either within $S$ or $S^c$ (i.e., $\mathbb{R}^d\setminus S$), and without loss of generality reorder the examples (and
permute the rows of $A$) so that $A :=\sbr{ \begin{smallmatrix}A_S\\A_c\end{smallmatrix}}$
where the rows of $A_S$ are within $S$ and the rows of $A_c$ are within $S^c$;
tying this to the earlier discussion, $A_c$ is the linearly separable part of the data,
and $A_S$ is the strongly convex part.
Furthermore, let $\Pi_S$ and $\Pip$ respectively denote orthogonal projection onto $S$ and $S^\perp$,
and define $A_\perp = \Pip A_c$, where each row of $A_c$ is orthogonally projected onto $S^\perp$.
By this notation,
\begin{align*}
    \Pip \nabla(L\circ A)(w) = \Pip \sbr{\begin{smallmatrix}A_c\\A_S\end{smallmatrix}}^\top
    \nL\del{ \sbr{\begin{smallmatrix}A_c\\A_S\end{smallmatrix}}w}
    = \Pip \sbr{\begin{smallmatrix}A_c\\A_S\end{smallmatrix}}^\top
    \sbr{\begin{smallmatrix}\nL(A_cw) \\\nabla L(A_Sw) \end{smallmatrix}}
     & = \sbr{\begin{smallmatrix}A_\perp\\0 \end{smallmatrix}}^\top
        \sbr{\begin{smallmatrix}\nL(A_cw) \\\nL(A_Sw) \end{smallmatrix}} \\
        &  = A_\perp^\top \nL(A_cw),
\end{align*}
which has made use of $L$ at varying input dimensions.

Gradient descent here will always start with $w_0 := 0$, and thereafter set
$w_{j+1} := w_j - \eta_j \nR(w_j)$.  It is convenient to define
$\gamma_j := |\nabla (\ln \cR)(w_j)| = |\nR(w_j)| / \cR(w_j)$
and $\heta_j := \eta_j \cR(w_j)$, whereby
\[
  |w_t| \leq \sum_{j<t} \eta_j |\nR(w_j)| = \sum_{j<t} \heta_j \gamma_j.
\]
Moreover, let $\barw_t := \argmin\cbr{ \cR(w) : |w|\leq |w_t| }$ denote the solution to the corresponding constrained optimization problem.

\section{Problem structure}
\label{sec:struct}

This section culminates in \Cref{fact:struct}, which characterizes
the unique ray $\{ \barv + r\cdot \baru : r\geq 0\}$.  To build towards this,
first consider the following examples.

\begin{wrapfigure}{R}{0.5\textwidth}
  \vspace{-0.5em}
  \includegraphics[width=0.5\textwidth]{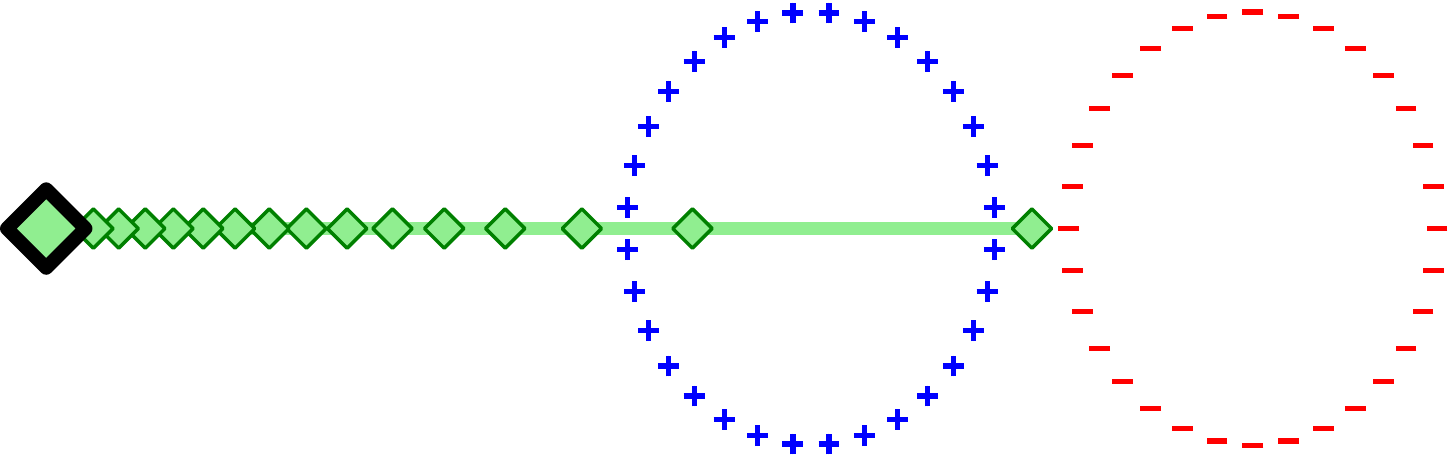}
  \caption{Separable.}
  \label{fig:struct:sep}
\end{wrapfigure}
\paragraph{Linearly separable.}
  Consider the data at right in \Cref{fig:struct:sep}: a blue circle of positive points, and a red circle of negative
  points.  The data is \emph{linearly separable:} there exist vectors $u\in\R^d$
  with positive margin, meaning $\min_i \ip{u}{x_iy_i} > 0$.  Taking any such $u$ and extending it to infinity
  will achieve 0 risk, but this is not what gradient descent chooses.  Constraining $u$ to have unit
  norm, a unique maximum margin point $\baru = -\bfe_1$ is obtained.
  The green gradient descent iterates follow $\baru$ exactly.

\begin{wrapfigure}{L}{0.4\textwidth}
  \vspace{-1.5em}
  \centerline{\includegraphics[width=0.25\textwidth]{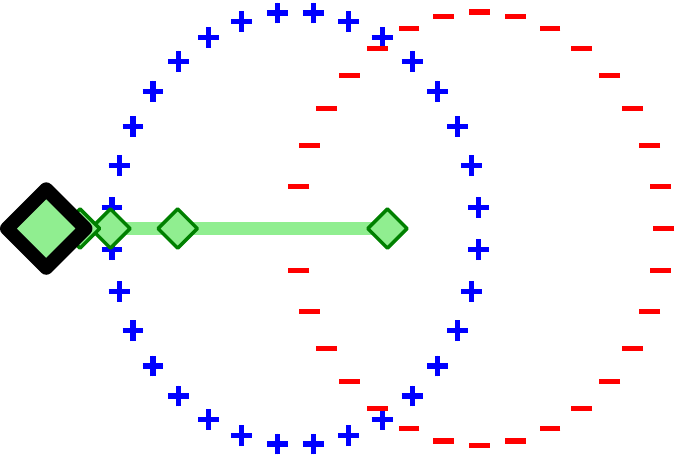}}
  \caption{Strongly convex.}
  \vspace{-1.0em}
  \label{fig:struct:sc}
\end{wrapfigure}
\paragraph{Strong convexity.}
Now consider moving the circles of data in \Cref{fig:struct:sep} until they overlap,
obtaining \Cref{fig:struct:sc}.  This data is \emph{not} linearly separable;
indeed, given any nonzero vector $u\in\R^d$, there exist data points incorrectly classified by
$u$, and therefore extending $u$ indefinitely will cause the risk to also increase to infinity.
It follows that the risk itself is 0-coercive \citep{HULL}, and moreover strongly convex
over bounded subsets, with a unique optimum $\barv$.  Gradient descent converges towards $\barv$.

\begin{wrapfigure}{R}{0.4\textwidth}
  \includegraphics[width=0.4\textwidth]{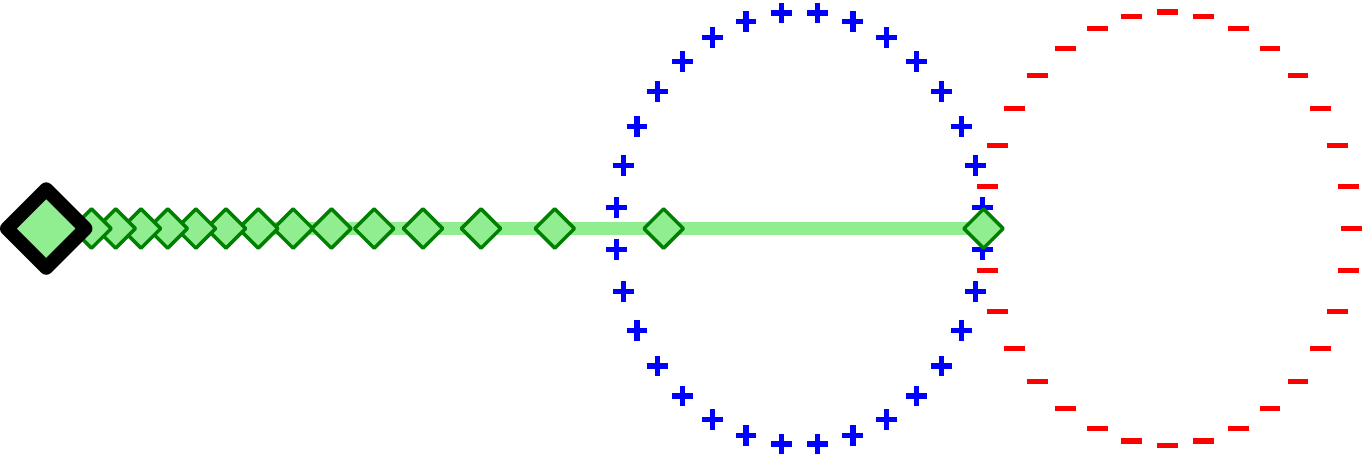}
  \caption{Mixed data.}
  \label{fig:struct:mixed:1}
\end{wrapfigure}
\paragraph{An intermediate setting.}
The preceding two settings either had the circles overlapping, or far apart.  What if they are
pressed together so that they touch at the origin?  Excluding the point at the origin,
the circles may still be separated with the maximum margin separator $\baru = -\bfe_1$ from
the linearly separable instance \Cref{fig:struct:sep}.
This example is our first taste of the general ray $\{\barv + r\cdot \baru:r\geq 0\}$, albeit still
with some triviality: $\barv = 0$.  Specifically, $\baru$ is the maximum margin separator of all data excluding the
point at the origin; the risk in this instance is bounded below by $\ell(0)/n$, which is the necessary error on the
point at the origin; the global optimum for that single point is $\barv=0$.

\begin{wrapfigure}{L}{0.5\textwidth}
  \includegraphics[width=0.5\textwidth]{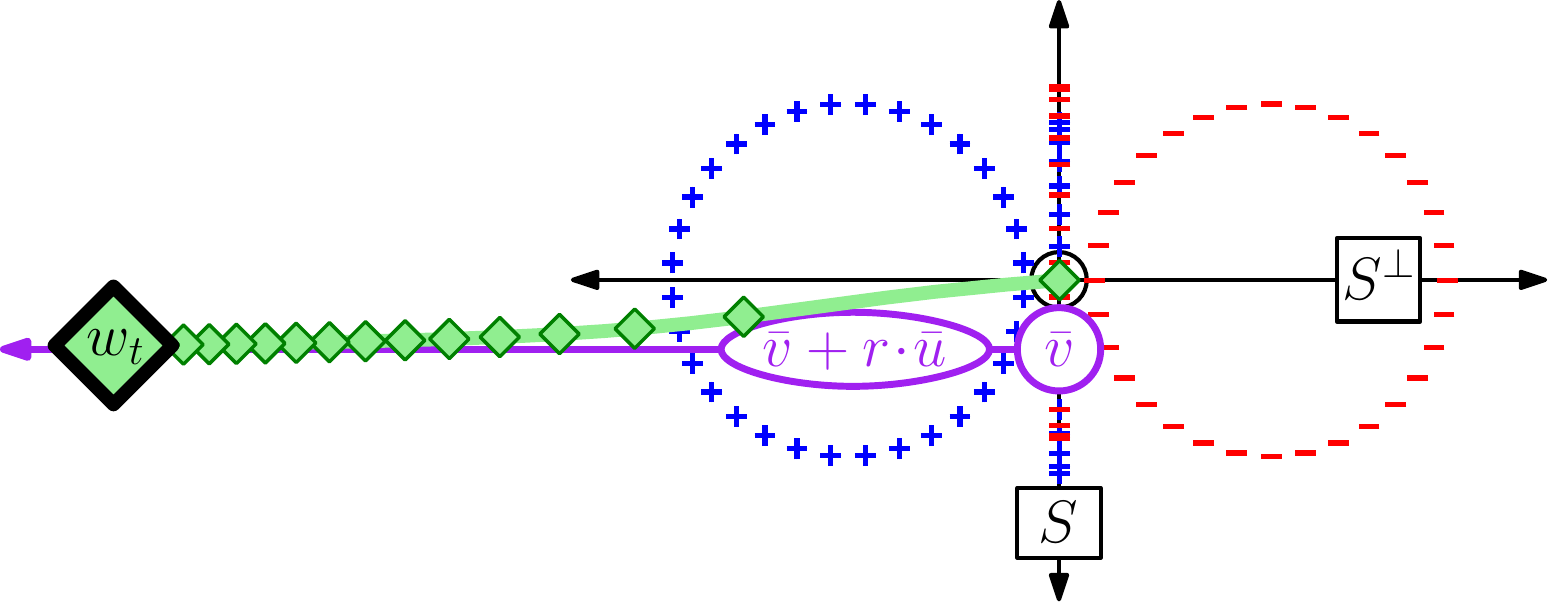}
  \caption{The general case.}
  \label{fig:struct:mixed:2}
\end{wrapfigure}
\paragraph{The general case.}
Combining elements from the preceding examples, the general case may be characterized as follows;
it appears in \Cref{fig:struct:mixed:2}, with all relevant objects labeled.
In the general case, the dataset consists of a \emph{maximal linearly separable subset $Z$}, with the remaining data falling
into a subset over which the empirical risk is strongly convex.  Specifically, $Z$ is constructed with the following
greedy procedure: for each example $(x_i,y_i)$, include it in $Z$ if there exists $u_i$ with $\ip{u_i}{x_iy_i}>0$
and $\min_j \ip{u_j}{x_jy_j}\geq 0$.  The aggregate $u := \sum_{i\in Z} u_i$ satisfies $\ip{u}{x_iy_i} >0$ for $i\in Z$
and $\ip{u}{x_iy_i} = 0$ otherwise.  Therefore $Z$ can be strictly separated by some vector $u$ orthogonal to $Z^c$;
let $\baru$ denote the maximum margin separator of $Z$ which is orthogonal to $Z^c$.

Turning now to $Z^c$, any vector $v$ which is correct on some $(x_i,y_i)\in Z^c$ (i.e., $\ip{u}{x_iy_i}>0$) must also
be incorrect on some other example $(x_j,y_j)$ in $Z^c$ (i.e., $\ip{v}{x_jy_j}<0$);
otherwise, $(x_i,y_i)$ would have been included in $Z$!  Consequently,
as in \Cref{fig:struct:sc} above, the empirical risk restricted to $Z^c$ is strongly convex, with a unique optimum $\barv$.  The gradient descent iterates follow the ray $\{\barv + r\cdot\baru:r\geq 0\}$, which means
they are globally optimal along $Z^c$, and achieve zero risk and follow the maximum margin direction $\baru$.

Turning back to the construction in \Cref{fig:struct:mixed:2}, the linearly separable data $Z$ is the two red and blue circles,
while $Z^c$ consists of data points on the vertical axis.  The points in $Z^c$ do not affect $\baru$, and have been adjusted to move $\barv$ away from 0,
where it rested in \Cref{fig:struct:mixed:1}.

Now using the notation $A_i=-y_ix_i^\top$, for $i\in Z$, the vector $-y_ix_i$ is collected into $A_c$, while for $i\in Z^c$, the vector $-y_ix_i^\top$ is put into $A_S$. The above constructions are made rigorous in the following \namecref{fact:struct}.
The proof of \Cref{fact:struct}, presented in the appendix, follows the intuition above.

\begin{theorem}
  \label{fact:struct}
  The rows of $A$ can be uniquely partitioned into matrices
  $(A_S, A_c)$,
  with a corresponding pair of orthogonal subspaces $(S,S^\perp)$ where
  $S = \SPAN(\AT_S)$
  satisfying the following properties.
  \begin{enumerate}
    \item\textbf{(Strongly convex part.)}
        If $\ell$ is twice continuously differentiable with $\ell''>0$,
        and $\ell \geq 0$,
        and $\lim_{z\to-\infty}\ell(z) = 0$,
        then $L\circ A$ is strongly convex over compact subsets of $S$,
        and $L\circ A_S$ admits a unique minimizer $\barv$
        with
        $
          \inf_{w\in\R^d} L(A w)
          =
          \inf_{v \in S}
          L(A_S v)
          = L(A_S\barv)
        $.

    \item\textbf{(Separable part.)}
      If $A_c$ is nonempty (and thus so is $A_\perp$),
      then $A_\perp$ is linearly separable.
      The maximum margin is given by
      \[
        \gamma
        := -\min\big\{ \max_{i} (A_\perp u)_i  : |u|= 1\big\}
        =\min\big\{|A_\perp^{\top}q|:q \geq 0, \sum_i q_i = 1\big\}
        > 0,
      \]
      and the maximum margin solution $\baru$ is the unique optimum to the primal problem,
      satisfying $\bar{u}=-A_\perp^{\top}\bar{q}/\gamma$ for every dual optimum $\barq$.
      If $\ell \geq 0$ and $\lim_{z\to-\infty}\ell(z) = 0$,
      then
      \begin{align*}
        \inf_{w\in\R^d} L(A w)=L(A_S \barv)+\lim_{r\to\infty}L\del{A_c(\barv+r\baru)}=L(A_S \barv).
      \end{align*}

  \end{enumerate}
\end{theorem}

\section{Risk convergence}
\label{sec:risk}

Gradient descent decreases the risk as follows.

\begin{theorem}\label{thm:risk_converge}
    For $\ell\in\cbr{\llog,\lexp}$, given step sizes $\eta_j\le1$ and $w_0=0$, then for any $t\ge1$,
    \begin{equation*}
      \cR(w_t)-\inf_w \cR(w)
      \le\frac{\exp\left(|\barv|\right)}{t}+\frac{|\barv|^2 + \ln(t)^2/\gamma^2}{2\sum_{j=0}^{t-1}\eta_j}.
    \end{equation*}
\end{theorem}

This proof relies upon three essential steps.
\begin{enumerate}
  \item
    A slight generalization of standard smoothness-based gradient
    descent bounds
    (cf. \Cref{lem:magic_ineq}).
  \item
    A useful comparison point to feed into the preceding gradient
    descent bound (cf. \Cref{lem:fixedDirRate}):
    the choice $\barv + \baru\del{\nicefrac{\ln(t)}{\gamma}}$,
    made possible by \Cref{fact:struct}.
  \item
    Smoothness estimates for $L$ when $\ell\in\{\llog,\lexp\}$
    (cf. \Cref{fact:smooth_ineq}).
\end{enumerate}

In more detail, the first step, a refined smoothness-based gradient
descent guarantee, is as follows.
While similar bounds are standard in the literature
\citep{bubeck,nesterov},
this version has a short proof and no issue with unbounded domains.

\begin{lemma}\label[lemma]{lem:magic_ineq}
  Suppose $f$ is convex, and there exists $\beta \geq 0$
  so that $1-\nicefrac{\eta_j\beta}{2}\ge0$ and gradient iterates $(w_0, \ldots, w_t)$ with $w_{j+1} := w_j - \eta_j \nabla f(w_j)$
  satisfy
  \[
    f(w_{j+1}) \leq f(w_j) - \eta_j\del{1 - \frac {\eta_j \beta} 2}\envert{ \nf(w_j) }^2.
  \]
  Then for any $z\in \mathbb{R}^d$,
\[
    2\sum_{j=0}^{t-1}\eta_j \del{ f(w_j) - f(z) }
    - \sum_{j=0}^{t-1} \frac {\eta_{j}}{1 - \nicefrac {\beta\eta_{j}}{2} }\del{ f(w_j) - f(w_{j+1}) }
    \leq
    \envert{w_{0} - z}^2 - \envert{ w_t - z }^2
    .
  \]
\end{lemma}

The proof is similar to the standard ones, and appears in the appendix.

The second step, as above, is to produce a reference point $z$
to plug into \Cref{lem:magic_ineq}.

\begin{lemma}\label[lemma]{lem:fixedDirRate}
  Let $\ell\in\{\lexp,\llog\}$ be given.
  Then $z := \barv + \baru\del{\nicefrac{\ln(t)}{\gamma}}$ satisfies
  $|z|^2 = |\barv|^2 + \nicefrac{\ln(t)^2}{\gamma^2}$ and
  \[
    \cR(z)
    \leq \inf_w \cR(w)+\frac {\exp(|\barv|)}{t}.
  \]
\end{lemma}
\begin{proof}
  By \Cref{fact:struct},
  and since $\llog\leq\lexp$ and $|A_i|\leq 1$,
  \begin{align*}
    L(Az)
    = L(A_S\barv)
    + L(A_c z)
    \leq \inf_w L(A w)
    + n\exp\del{|\barv|-\ln(t)}
    = \inf_w L(A w)
    + \frac{n\exp\del{|\barv|}}{t}.
  \end{align*}
\end{proof}

Lastly, the smoothness guarantee on $\cR$.
Even though the logistic loss is smooth, this proof gives a refined smoothness
inequality where the $j^{\textup{th}}$ step is $\cR(w_j)$-smooth;
this refinement will be essential when proving parameter convergence.
This proof is based on the convergence guarantee for AdaBoost \citep{schapire_freund_book_final}.
Recall the definitions $\gamma_j := |\nabla (\ln \cR)(w_j)| = |\nR(w_j)| / \cR(w_j)$
and $\heta_j := \eta_j \cR(w_j)$.
\begin{lemma}
  \label[lemma]{fact:smooth_ineq}
  Suppose $\ell$ is convex,
  $\ell'\leq \ell$,
  $\ell''\leq \ell$,
  and $\heta_{j} = \eta_{j} \cR(w_j)\leq 1$.
  Then
  \[
    \cR(w_{j+1}) \leq
    \cR(w_j)
    - \eta_j\del{1 - \frac {\eta_j \cR(w_j)}{2}}|\nR(w_j)|^2
    = \cR(w_j)\del{1-
      \heta_j(1-\heta_j/2)\gamma_j^2
    }
  \]
  and thus
  \begin{align*}
    \cR(w_t)
    &\leq
    \cR(w_0) \prod_{j<t} \del{ 1 - \heta_j (1 - \heta_j/2)\gamma_j^2 }
    \leq
    \cR(w_0)\exp\del{ -\sum_{j < t} \heta_j (1 - \heta_j/2)\gamma_j^2 }.
  \end{align*}
  Additionally, $|w_t|\leq \sum_{j<t} \heta_j \gamma_j$.
\end{lemma}

This proof mostly proceeds in a usual way via recursive application of a Taylor expansion
\begin{align*}
  \cR\left(w-\eta\nR(w)\right)
  &\leq \cR(w) - \eta |\nR(w)|^2
  +
  \frac 1 2 \max_{v\in [w,w']} \sum_i (A_i(w-w'))^2 \ell'' (A_iv) /n.
\end{align*}
The interesting part is the inequality $\ell''\leq \ell$: this allows the final term in
the preceding expression to replace $\ell''$ with $\ell$, and some massaging with the maximum
allows this term to be replaced with $\cR(w)$ (since a descent direction was followed).

A direct but important consequence is obtained by applying $\ln$ to both sides:
\begin{equation}
    \ln \cR(w_t)\le \ln \cR(w_0)-\sum_{j < t} \heta_j (1 - \heta_j/2)\gamma_j^2.
    \label{eq:lnR}
\end{equation}
It follows that $\ln \cR$ is smooth,
but moreover has \emph{constant} smoothness unlike $\cR$ above.

Note that for $\ell\in\cbr{\llog,\lexp}$, since $\cR(w_0)\le1$, choosing $\eta_j\le1$ will ensure $\hat{\eta}_j\le1$, and then \Cref{fact:smooth_ineq} holds. Also, \Cref{fact:smooth_ineq} shows that \Cref{lem:magic_ineq} holds for $\cbr{\cR_{\exp},\cR_{\log}}$ with $\beta=1$.

Combining these pieces now leads to a proof of \Cref{thm:risk_converge}, given in full in the appendix.
As a final remark, note that step size $\eta_j = 1$ led to a $\widetilde{\cO}(1/t)$ rate,
whereas $\eta_j = 1/\sqrt{j+1}$ leads to a $\widetilde{\cO}(1/\sqrt{t})$ rate.

\section{Parameter convergence}
\label{sec:param}

As in \Cref{fact:main}, the parameter convergence guarantee
gives convergence to $\barv\in S$ over the strongly convex part $S$
(that is, $\Pi_S w_t \to \barv$ and $\Pi_S\barw_t \to \barv$),
and convergence \emph{in direction}
(convergence of the normalized iterates) to $\baru\in S^\perp$
over the separable part $S^\perp$
($\nicefrac{w_t}{|w_t|} \to \baru$
and $\nicefrac{\barw_t}{|\barw_t|} \to \baru$).
In more detail, the convergence rates are as follows.

\begin{theorem}
  \label{fact:min_norm}
  Let loss $\ell\in\cbr{\lexp,\llog}$ be given.
  \begin{enumerate}
    \item \textbf{(Separable case.)}
      Suppose $A$ is separable (thus $S^\perp = \R^d$,
      and $A = A_c = A_\perp$, and $A\baru = A_\perp\baru \leq -\gamma$).
      Furthermore,
      suppose $\eta_j = 1$,
      and $t\geq 5$,
      and
        $\displaystyle
        \nicefrac{t}{\ln(t)^3}\ge \nicefrac{n}{\gamma^4}.
      $
      Then
      \[
        \max\cbr{
          \envert{\frac{w_t}{|w_t|}-\bar{u}}^2,
          \,
          \envert{\frac{\bar{w}_t}{|\bar{w}_t|}-\bar{u}}^2
        }
        = \cO\del{
          \frac{\ln n+\ln\left(|w_t|\right)}{|w_t|\gamma^2}
        }
        = \cO\del{
          \frac{\ln n+\ln\ln t}{\gamma^2 \ln t }
        }.
      \]

    \item \textbf{(General case.)}
      Suppose $\eta_j=1/\sqrt{j+1}$,
      and $t\geq 5$,
      and
        $\displaystyle
        \nicefrac{\sqrt{t}}{\ln(t)^3}\ge \nicefrac{n(1+R)}{\gamma^2},
      $
      where $R : =\sup_{j<t}|\Pi_S w_j|=\cO(1)$.
      Then
      \begin{align*}
        |\Pi_S w_t| = \Theta\del{ 1 }
        \quad&\textup{and}\quad
        \max\cbr{ |\Pi_S \barw_t - \barv|^2, |\Pi_S \barw_t - \barv|^2 } = \cO\del{\frac{\ln(t)^2}{\sqrt{t}}},
      \end{align*}
      and if $A_c$ is nonempty, then $|\Pi_{S^\perp} w_t| = \Theta(\ln(t))$, and
      \begin{align*}
          \max\cbr{ \envert{\frac{w_t}{|w_t|} - \baru }^2,
                    \envert{\frac{\barw_t}{|\barw_t|} - \baru }^2 }
          = \cO\del{\frac{\ln n + \ln |w_t|}{|w_t|\gamma^2}}
          = \cO\del{\frac{\ln n + \ln \ln|t|}{\gamma^2\ln(t)}}.
      \end{align*}
  \end{enumerate}
\end{theorem}

Establishing rates of parameter convergence not only relies upon all previous sections,
but also is much more involved, and will be split into multiple subsections.

The easiest place to start the analysis is to dispense with the behavior over $S$.
For convenience, define
\begin{equation*}
    \cR_S(w) := \frac{L(A_Sw)}{n},
    \quad
    \cR_c(w) := \frac{L(A_cw)}{n},
    \quad
    \bar{\cR} := \inf_{w\in \mathbb{R}^d}\cR(w),
\end{equation*}
and note $\cR(w) =\cR_S(w)+\cR_c(w)$.

Convergence over $S$ is a consequence of strong convexity and
risk convergence (cf. \Cref{thm:risk_converge}).
\begin{lemma}\label[lemma]{fact:param_converge_S}
    Let $\ell\in\{\lexp,\llog\}$, and $\lambda$ denote the modulus of strong convexity of
    $\cR_S$ over the $1$-sublevel set (guaranteed positive by \Cref{fact:struct}).
    Then for any $t\ge1$,
    \begin{equation*}
      \max\cbr{ |\Pi_S w_t -\bar{v}|^2,\ {} |\Pi_S \barw_t - \barv|^2 }
      \leq \frac 2 \lambda \min\cbr{
        1,\ {}
        \frac{\exp\left(|\barv|\right)}{t}+\frac{|\barv|^2 + \ln(t)^2/\gamma^2}{2\sum_{j=0}^{t-1}\eta_j}.
    }
    \end{equation*}
\end{lemma}
\begin{proof}
  By \Cref{fact:struct}, $\cR_S(\barv) = \bar\cR$.
  Thus, by strong convexity, for $w \in \cbr{w_t,\barw_t}$ (whereby $\cR(w) \leq \cR(w_t)$),
  \[
    |\Pi_S w -\bar{v}|^2\le \frac{2}{\lambda}\left(\cR_S(w)-\cR_S(\bar{v})\right)\le \frac{2}{\lambda}\left(\cR(w_t)-\inf_{w\in \mathbb{R}^d}\cR(w)\right).
  \]
  The bound follows by noting $\cR(w_t)\leq \cR(w_0) \leq 1$, and alternatively invoking in \Cref{thm:risk_converge}.
\end{proof}

If $A_c$ is empty, the proof is complete by plugging $\eta_j=1/\sqrt{j+1}$ into \Cref{fact:param_converge_S}.
The rest of this section establishes convergence to $\baru \in S^\perp$.

\subsection{Bounding $|\Pi_\perp w_t|$}

Before getting into the guts of \Cref{fact:min_norm},
this section will establish bounds on $|w_t|$.  These bounds are used in
\Cref{fact:min_norm} in two ways.  First, as will be clear in the next section,
it is natural to prove rates with $|w_t|$ in the denominator,
thus lower bounding $|w_t|$ with a function of $t$ gives the desired bound.

On the other hand, it will be necessary to also produce an upper bound
since the proofs will need a warm start,
which will place a $|w_{t_0}|$ in the numerator.  The need for this warm
start will be discussed in subsequent sections,
but a key purpose is to make $\lexp$ and $\llog$ appear more similar.

Note that this section focuses on behavior within $A_c$;
suppose that $A_c$ is nonempty, since otherwise $A = A_S$
and \Cref{fact:min_norm} follows from
\Cref{fact:param_converge_S}.
When $A_c$ is nonempty, the solution is off at infinity, and $|w_t|$
grows without bound.  For this reason, these bounds will be on
$w_t$ rather than $\Pip w_t$,
since $|w_t - \Pip w_t| = |\Pi_S w_t| = \cO(1)$
by \Cref{fact:param_converge_S}.

\begin{lemma}
  \label[lemma]{fact:wt_lnt}
  Suppose $A_c$ is nonempty,
  $\ell\in\left\{\lexp,\llog\right\}$, and step sizes satisfy $\eta_j\le1$.
  Define $R : =\sup_{j<t}|\Pip w_j-w_j|$, where \Cref{fact:param_converge_S}
  guarantees $R = \cO(1)$.
  Let $n_c > 0$ denote the number of rows in $A_c$.
  For any $t\ge1$,
  \begin{align*}
    |w_t|
    &\le \max\cbr{\frac{4\ln(t)}{\gamma^2},\ {}\frac{4R}{\gamma^2},\ 2},
    \\
    |w_t|
    &\ge \min\Big\{ \ln(t) -\ln(2) - |\barv|,\ {} \ln \Big(\sum_{j=0}^{t-1}\eta_j\Big) - \ln\Big(|\barv|^2 + \frac{ \ln(t)^2}{\gamma^2} \Big) \Big\}
    -R + \ln \ln(2) - \ln\Big(\frac{n}{n_c}\Big).
  \end{align*}
\end{lemma}

The proof of \Cref{fact:wt_lnt} is involved, and analyzes a key quantity which is extracted
from the perceptron convergence proof \citep{novikoff}. Specifically, the perceptron convergence
proof establishes bounds on the number of mistakes up through time $t$ by upper and lower bounding $\ip{w_t}{\baru}$.
As the perceptron algorithm is stochastic gradient descent on the ReLU loss $z\mapsto \max\{0,z\}$,
the number of mistakes up through time $t$ is more generally the quantity $\sum_{j<t} \ell'(\ip{w_j}{-x_jy_j})$.
Generalizing this, the proof here controls the quantity
$
  \ell'_{<t} := \frac 1 n \sum_{j<t} \eta_j |\nabla L(A_c w_j)|_1
$
.
Rather than obtaining bounds on this by studying $\ip{w_t}{\baru}$,
the proof here works with $\ip{w_t - \baru\cdot r}{\baru}=\ip{\Pip w_t - \baru\cdot r}{\baru}$, with the strategic choice $r:= \ln(t) / \gamma$.

On one hand, $\ip{\Pip w_t - \baru\cdot r}{\baru}$ is upper bounded by $\enVert{\Pip w_t - \baru\cdot r}$,
whose upper bound is given by the following lemma, which is a modification of \Cref{lem:magic_ineq}
with $\cR$ replaced with $\cR_c$.

\begin{lemma}
  \label[lemma]{lem:magic_ineq_perp}
  Let any $\ell\in\left\{\lexp,\llog\right\}$ be given,
  and suppose $\eta_j \leq 1$.
  For any $u\in S^\perp$ and $t\geq 1$,
  \[
    \envert{ \Pip w_t - u}^2 \le |u|^2+2+\sum_{j<t}^{}2\eta_j\left(\cR_c(u) - \cR_c(w_j)\right)+\sum_{j<t}^{}2\eta_j\ip{\nR_c(w_j)}{w_j - \Pip w_j}.
  \]
\end{lemma}

The proof is similar to that of \Cref{lem:magic_ineq},
but inserts $\Pi_\perp$ in a few key places.

On the other hand, to lower bound $\ip{\Pip w_t - \baru\cdot r}{\baru}$, notice that
\begin{align*}
  \ip{ \Pip w_t}{\baru}
  &=
  \frac 1 n \ip{ -\Pip \sum_{j<t} \eta_j A^\top \nL(Aw_j)}{\baru}
  =
  \frac 1 n \ip{ -\sum_{j<t} \eta_j A_\perp^\top \nL(A_cw_j)}{\baru}
  \\
  &=
  \sum_{j<t} \frac{\eta_j |\nL(A_cw_j)|_1}{n} \ip{ - A_\perp^\top \frac{\nL(A_cw_j)}{|\nL(A_c w_j)}|_1}{\baru}
  \geq
  \sum_{j<t} \frac{\gamma \eta_j |\nL(A_cw_j)|_1}{n}
  = \gamma \ell'_{<t},
\end{align*}
where the last step uses the definition of $\ell'_{<t}$ from above.
After a fairly large amount of careful but uninspired manipulations,
\Cref{fact:wt_lnt} follows.
A key issue throughout this and subsequent subsections is the handling
of cross terms between $A_c$ and $A_S$.

\subsection{Parameter convergence over $S^\perp$: separable case}

First consider the simple case where the data is separable;
in other words, $S^\perp = \R^d$ and $A = A_c = A_\perp$.
The analysis in this case hinges upon two ideas.
\begin{enumerate}
  \item
    To show that $w$ is close to $\baru$ in direction, it is essential to lower bound $\ip{\baru}{w}/|w|$,
    since
    \[
      \envert{ \baru - \nicefrac{w}{|w|} }^2 = 2 - \frac{2\ip{\baru}{w}}{|w|}.
    \]
    To control this, recall the representation $\baru = -\AT\barq/\gamma$,
    where $\barq$ is a dual optimum (cf. \Cref{fact:struct}).
    With this in hand, and an appropriate choice of convex function $g$,
    the Fenchel-Young inequality gives for any $w_t$, $t\ge1$, and any $w$ such that $g(Aw)\le g(Aw_t)$,
    \begin{equation}
      - \frac {\ip{\baru}{w}}{|w|}
      = \frac {\ip{\barq}{Aw}}{\gamma|w|}
      \leq \frac {g^*(\barq) + g(Aw)}{\gamma|w|}
      \leq \frac {g^*(\barq) + g(Aw_t)}{\gamma|w|}.
      \label{eq:magic_F-Y}
    \end{equation}
    The significance of working with $w$ with $g(Aw)\leq g(Aw_t)$
    is to allow the proof to handle both $\barw_t$ and $w_t$ simultaneously.
  \item
    The other key idea is to use $g = \ln( L_{\exp} / n )$.
    With this choice, both preceding numerator terms can be bounded:
    $g^*(q)=\ln n+\sum_{i=1}^{n}q_i\ln q_i\le\ln n$ for any probability vector $q$,
    whereas $g(Aw_t)$ can be upper bounded by applying $\ln$ to both sides of
    \Cref{fact:smooth_ineq}, as \cref{eq:lnR},
    yielding an expression which will cancel with the denominator
    since $|w_t|\leq \sum_{j<t} \heta_j \gamma_j$.
\end{enumerate}

For illustration, suppose temporarily that $\ell=\lexp$. Still let $w_t$ be any gradient descent iterate with $t\ge1$, and $w$ satisfy $g(Aw)\le g(Aw_t)$.
Continuing as in \cref{eq:magic_F-Y}, but also assuming $\eta_j\leq 1$ and invoking \Cref{fact:smooth_ineq}
to control $g(Aw_t) = \ln\cR(w_t)$ and noting $g^*(\barq)\leq \ln(n)$, $\cR(w_0)=1$,
\begin{align*}
  \frac{1}{2}\left|\frac{w}{|w|}-\bar{u}\right|^2 & \le1+\frac{\ln\cR(w_t)}{|w|\gamma}+\frac{\ln(n)}{|w|\gamma} \\
     & \le1+\frac{\ln\cR(w_0)}{|w|\gamma}-\frac{\sum_{j=0}^{t-1}\hat{\eta}_{j}(1-\hat{\eta}_{j}/2)\gamma_{j}^2}{|w|\gamma}+\frac{\ln(n)}{|w|\gamma} \\
     &= 1-\frac{\sum_{j=0}^{t-1}\hat{\eta}_{j}\gamma_{j}^2}{|w|\gamma}+\frac{\sum_{j=0}^{t-1}\hat{\eta}_{j}^2\gamma_{j}^2}{2|w|\gamma}+\frac{\ln(n)}{|w|\gamma}.
\end{align*}
To simplify this further, note $\eta_j\leq 1$ and \Cref{fact:smooth_ineq} also imply
\[
    \sum_{j=t_0}^{t-1}\hat{\eta}_{j}^2\gamma_{j}^2=\sum_{j=t_0}^{t-1}\eta_{j}^2|\nR(w_j)|^2\le2\sum_{j=t_0}^{t-1}(\cR(w_{j})-\cR(w_{j+1}))\le2,
\]
and moreover the definition $\gamma = \min\cbr{ |A_\perp^\top q | : q \geq 0, \sum_i q = 1 }$ (cf. \Cref{fact:struct})
implies
\[
  \gamma_j=|\nabla(\ln\cR)(w_j)|=\frac{|\AT\nL(w_j)|}{L(Aw_j)}\ge\gamma.
\]
Combining these simplifications,
\begin{align*}
  \frac{1}{2}\left|\frac{w}{|w|}-\bar{u}\right|^2
     & \le1-\frac{\sum_{j=0}^{t-1}\hat{\eta}_{j}\gamma_{j}\gamma}{|w|\gamma}+\frac{2}{2|w|\gamma}+\frac{\ln(n)}{|w|\gamma}
     \le \frac{1+\ln n}{|w|\gamma}.
\end{align*}
To finish, invoke the preceding inequality with $w\in\{\barw_t, w_t\}$, noting that $|w_t|=|\barw_t|$.
To produce a rate depending on $t$ and not $|w_t|$, the lower bound on $|w_t|$ in \Cref{fact:wt_lnt}
is applied with $|\barv|=R=0$ and $n=n_c$ thanks to separability.
The following proposition summarizes this derivation.
\begin{proposition}
  Suppose $\ell = \lexp$ and $A = A_c = A_\perp$ (the separable case).
  Then, for all $t\geq 1$,
  \[
    \max\left\{\left|\frac{w_t}{|w_t|}-\bar{u}\right|^2,\left|\frac{\barw_t}{|\barw_t|}-\bar{u}\right|^2\right\}
     \le \frac{2+ 2\ln n}{|w_t|\gamma},
   \]
   where $|w_t| \geq \min\cbr{\ln(t)-\ln 2,\ {} \ln\del{\sum_{j<t} \eta_j} - 2\ln\ln t + 2\ln \gamma}+\ln\ln 2$.
\end{proposition}

Adjusting this derivation for $\llog$ is clumsy, and will only be sketched here, instead mostly appearing
in the appendix.  The proof will still follow the same scheme, even defining $g = \ln \cR_{\exp}$
(rather than $g = \ln\cR$), and will use the following \namecref{fact:log_approx}
to control $\ell$ and $\lexp$ simultaneously.
\begin{lemma}
  \label[lemma]{fact:log_approx}
  For any $0<\epsilon\le1$, $\ell\in\left\{\lexp,\llog\right\}$, if $\ell(z)\le\epsilon$, then
  \[
    \frac{\ell'(z)}{\ell(z)}\ge1-\epsilon
    \qquad\textup{and}\qquad
    \frac{\lexp(z)}{\ell(z)}\le2.
  \]
\end{lemma}

The general Fenchel-Young scheme from above can now be adjusted.
The bound requires $\cR(w_t)\leq \eps/n$,
which implies $\max_i \ell((Aw_t)_i) \leq \eps$,
meaning each point has some good margin,
rendering $\llog$ and $\lexp$ similar.

\begin{lemma}\label[lemma]{fact:ip}
    Let $\ell\in\left\{\lexp,\llog\right\}$. For any $0<\epsilon\le1$,
    and any $t$ with $\cR(w_t)\leq \eps/n$, and any $w$ with $\cR(w)\leq \cR(w_t)$,
    \[
      \frac{\ip{\baru}{w}}{|\baru|\cdot|w|}\geq-\frac{\ln\cR(w_t)}{|w|\gamma}-\frac{g^*(\bar{q})}{|w|\gamma}-\frac{\ln2}{|w|\gamma}.
    \]
\end{lemma}
\begin{proof}
    By the condition, $\cR(w)\le\cR(w_t)\le\epsilon/n$, and thus for any $1\le i\le n$, $\ell(A_iw)\le\epsilon\le1$. By \Cref{fact:log_approx}, $\cR_{\exp}(w)\le2\cR(w)$.
    Then by \Cref{fact:struct} and the Fenchel-Young inequality,
    \begin{align*}
        \frac{\ip{\baru}{w}}{|\baru|\cdot|w|}
        & =-\frac {\ip{\barq}{A w}}{|w|\gamma}
         \ge-\frac{g^*(\bar{q})}{|w|\gamma}-\frac{g(Aw)}{|w|\gamma}
         =-\frac{g^*(\bar{q})}{|w|\gamma}-\frac{\ln\cR_{\exp}(w)}{|w|\gamma}
         \\
         &
         \ge-\frac{g^*(\bar{q})}{|w|\gamma}-\frac{\ln2}{|w|\gamma}-\frac{\ln\cR(w)}{|w|\gamma}
         \ge-\frac{g^*(\bar{q})}{|w|\gamma}-\frac{\ln2}{|w|\gamma}-\frac{\ln\cR(w_t)}{|w|\gamma}.
    \end{align*}
\end{proof}

Proceeding as with the earlier $\lexp$ derivation,
since $g^*(\barq) \leq \ln(n)$,
and using the upper bound on $\ln\cR(w_t)$ from \Cref{fact:smooth_ineq},
\[
  \frac{\ip{\baru}{w}}{|\baru|\cdot|w|}
  \ge-\frac{\ln(n)}{|w|\gamma}-\frac{\ln2}{|w|\gamma}-\frac{\ln\cR(w_{t_0}) - \sum_{j=t_0}^{t-1} \heta_j(1-\heta_j/2)\gamma_j^2}{|w|\gamma}.
\]
The role of the warm start $t_0$ here is to make $\lexp$ and $\llog$ behave similarly,
in concrete terms allowing the application of \Cref{fact:log_approx} for $j\geq t_0$.
For instance, the earlier $\lexp$ proof used $\gamma_j \geq \gamma$, but this now becomes $\gamma_j\geq (1-\eps)\gamma$.
Completing the proof of the separable case of \Cref{fact:min_norm} requires many such considerations which did not appear with $\lexp$,
including an upper bound on $|w_{t_0}|$ via \Cref{fact:wt_lnt}.

\subsection{Parameter convergence over $S^{\perp}$: general case}

The scheme from the separable case does not directly work:
for instance, the proofs relied upon $\gamma_i \geq \gamma$, but now $\gamma_i \to 0$.
This term $\gamma_i$ arose by applying \Cref{fact:smooth_ineq} to control $\ln\cR(w_t)$,
which in the separable case decreased to $-\infty$, in the general case, however, it can
be bounded below.

The fix is to replace $\cR(w_t)$ with $\cR_c(w_t)$,
or rather $\cR(w_t) - \bar \cR = \cR(w_t) - \inf_w \cR(w)$;
this quantity goes to 0, and there is again a hope of exhibiting the fortuitous cancellations
which proved parameter convergence.
More abstractly, by subtracting $\bar\cR$, the proof is again trying to work
in the separable case, though there will be cross-terms to contend with.

The first step, then, is to replace the appearance of $\cR$ in earlier Fenchel-Young approach (cf. \Cref{fact:ip})
with $\cR(w_t) - \bar \cR$.

\begin{lemma}\label[lemma]{fact:gen:ip}
    Let $\ell\in\left\{\lexp,\llog\right\}$. For any $0<\epsilon\le1$, $t\ge1$,
    and any $w$ with $\cR(w) - \bar\cR \leq \cR(w_t)-\bar\cR \leq \eps/n$,
    \[
    \frac{\ip{\baru}{w}}{|\baru|\cdot|w|}\geq
    \frac{-\ln \del{ \cR(w_t) -\bar{\cR} }}{\gamma|w|}-\frac{\ln2+g^*(\barq)+|\Pi_S(w)|}{\gamma|w|}.
    \]
\end{lemma}

Note the appearance of the additional cross term $|\Pi_S(w)|$;
by \Cref{fact:param_converge_S}, this term is bounded.

The next difficulty is to replace the separable case's use of \Cref{fact:smooth_ineq}
to control $\ln\cR(w_t)$ with something controlling $\ln(\cR(w_t) - \bar\cR)$.

\begin{lemma}
  \label[lemma]{fact:gen:iter}
  Suppose $\ell\in\{\llog,\lexp\}$ and $\heta_j \leq 1$ (meaning $\eta_{j} \leq 1 / \cR(w_j)$).
  Also suppose that $j$ is large enough such that $\cR(w_j)-\bar{\cR}\leq\min\left\{\epsilon/n,\lambda (1-r)/2\right\}$
  for some $\epsilon,r\in(0,1)$, where $\lambda$ is the strong convexity modulus of $\cR_S$ over the $1$-sublevel set.
  Then
  \begin{align*}
    \cR(w_{j+1}) - \bar{\cR}
    &\leq \del{ \cR(w_j) - \bar{\cR} }
    \exp\del{ -r(1-\epsilon)\gamma\gamma_j\heta_j\del{1-\heta_j/2} }.
  \end{align*}
  Moreover, if there exists a sequence $(w_j)_{j=t_0}^{t-1}$ such that the above condition holds,
  then
  \[
    \cR(w_t) - \bar{\cR}
    \leq \del{ \cR(w_{t_0}) - \bar{\cR} }
    \exp\del{
      - r(1-\epsilon)\gamma \sum_{j = t_0}^{t-1} \heta_j \del{1-\heta_j/2} \gamma_j
    }.
  \]
\end{lemma}
The proof of \Cref{fact:gen:iter} is quite involved,
but boils down to the following case analysis.
The first case is that there is more error over $S$, whereby a strong convexity argument gives a lower bound
on the gradient.  Otherwise, the error is larger over $S^c$, which leads to a big step in the direction of $\baru$.

A key property of the upper bound in \Cref{fact:gen:iter} is that it has replaced $\gamma_j^2$ in \Cref{fact:smooth_ineq}
with $\gamma_j \gamma$.
Plugging this bound into the Fenchel-Young scheme in \Cref{fact:gen:ip} will now fortuitously cancel $\gamma$,
which leads to the following promising bound.

\begin{lemma}
  \label[lemma]{fact:gen:ip:meh}
  Let $\ell\in\cbr{\llog,\lexp}$ and
  $0<\epsilon\le1$ be given,
  and select $t_0$ so that $\cR(w_{t_0})-\bar{\cR} \leq \eps/n$.
  Then for any $t\ge t_0$ and any $w$ such that $\cR(w)\le\cR(w_t)$,
  \[
    \frac{\ip{\baru}{w}}{|\baru|\cdot|w|}
    \geq
    \frac{r(1-\epsilon)\sum_{j = t_0}^{t-1} \heta_j (1-\heta_j/2) \gamma_j}{|w|}-\frac{\ln2}{\gamma|w|}-\frac{|\Pi_S(w)|}{\gamma|w|}.
  \]
\end{lemma}

As in the separable case, an extended amount of careful massaging,
together with \Cref{fact:smooth_ineq} and \Cref{fact:wt_lnt} to control the warm start,
is sufficient to establish the parameter convergence of \Cref{fact:min_norm} in the general case.

\section{Related work}\label{sec:related}

The technical basis for this work is drawn from the literature on AdaBoost,
which was originally stated for separable data \citep{freund_schapire_adaboost},
but later adapted to general instances
\citep{mukherjee_rudin_schapire_adaboost_convergence_rate,primal_dual_boosting}.
This analysis revealed not only a problem structure which can be refined into the
$(S,S^\perp)$ used here, but also the convergence to maximum margin solutions
\citep{mjt_margins}.  Since the structural analysis is independent of the optimization
method, the key structural result, \Cref{fact:struct} in \Cref{sec:struct},
can be partially found in prior work;
the present version provides not only an elementary proof, but moreover
differs by providing $S$ and $S^\perp$ (and not just a partition of the data)
and the subsequent construction of a unique $\baru$ and its properties.

The remainder of the analysis has some connections to the AdaBoost literature,
for instance when providing smoothness inequalities for $\cR$ (cf. \Cref{fact:smooth_ineq}).
There are also some tools borrowed from the convex optimization literature,
for instance smoothness-based convergence proofs of gradient descent \citep{nesterov,bubeck},
and also from basic learning theory, namely an adaptation of ideas from the perceptron
convergence proof in order to bound $|w_t|$ \citep{novikoff}.

Another close line of work is an analysis of gradient descent for logistic
regression when the data is separable \citep{nati_logistic,GLSS18,NLGSS18}.
The analysis is conceptually different (tracking $(w_j)_{j\geq 0}$ in
all directions, rather than the Fenchel-Young and smoothness approach here), and
(assuming linear separability) achieves a better rate than the one here,
although it is not clear if this is possible in the nonseparable case.
Other work shows that not just gradient descent but also steepest descent with other norms
can lead to margin maximization \citep{GLSS18,mjt_margins},
and that constructing loss functions with an explicit goal of margin maximization
can lead to better rates \citep{NLGSS18}.
Another line of work uses condition numbers to analyze these problems
with a different parameterization \citep{rfreund}.

There has been a large literature on risk convergence of logistic regression.
\citep{hazan_online,mahdavi} assume a bounded domain,
exploit the exponential concavity of the logistic loss
and apply online Newton step to give a $\cO(1/t)$ rate.
However, on a domain with norm bound $D$, the exponential concavity term is $1/\beta = e^D$.
In the unbounded setting considered in this paper,
this term is a factor of $\mathrm{poly}(t)$
since $|w_t|=\Theta(\ln(t))$ (cf. \Cref{fact:wt_lnt}).
\citep{bach2013non,bach_logistic} assume optima exist, and
the rates depend inversely on the smallest eigenvalue of the Hessian at optima,
which similarly introduce a factor of $\mathrm{poly}(t)$ when there is no finite optimum.

There is some work in online learning on optimization over unbounded sets,
for instance bounds where the regret scales with the norm of the comparator
\citep{orabona__coin_betting_unbounded,streeter__noregret_unconstrained}.
By contrast, as the present work is not adversarial and instead has a fixed training
set, part of the work (a consequence of the structural result, \Cref{fact:struct})
is the existence of a good, small comparator.

\subsection*{Acknowledgements}

The authors are grateful for support from the NSF under grant
IIS-1750051.

\bibliography{../bib}
\bibliographystyle{plainnat}

\appendix

\section{Omitted proofs from \Cref{sec:struct}}

Before proving \Cref{fact:struct},
note the following result characterizing margin maximization over $S^\perp$.
\begin{lemma}
  \label[lemma]{fact:margin_duality}
  Suppose $A_\perp$ has $n_c>0$ rows and there exists $u$ with $A_\perp u < 0$.  Then
  \[
    \gamma
    := -\min\big\{ \max_{i} (A_\perp u)_i  : |u|= 1\big\}
    =\min\big\{|A_\perp^{\top}q|:q \geq 0, \sum_i q_i = 1 \big\}
    > 0.
  \]
  Moreover there exists a unique nonzero primal optimum $\baru$,
  and every dual optimum $\bar{q}$ satisfies $\bar{u}=-A_\perp^{\top}\bar{q}/\gamma$.
\end{lemma}
\begin{proof}
  To start, note $\gamma > 0$ since there exists $u$ with $A_\perp u < 0$.

  Continuing, for convenience define simplex $\Delta := \{ q \in\R^{n_c} : q\geq 0, \sum_i q_i = 1\}$,
  and convex indicator $\iota_K(z) = \infty\cdot\mathds{1}[z\in K]$.
  With this notation, note the Fenchel conjugates
  \begin{align*}
    \iota_{\Delta}^*(v)
    &= \sup_{q\in{\Delta}} \ip{v}{q} = \max_i v_i,\\
    (|\cdot|_2)^*(q)
    &= \iota_{|\cdot|_2 \leq 1}(q).
  \end{align*}
  Combining this with the Fenchel-Rockafellar duality theorem
  \citep[Theorem 3.3.5, Exercise 3.3.9.f]{borwein_lewis},
  \begin{align*}
    \min |A_\perp^\top q|_2 + \iota_\Delta(q)
    &= \max - \iota_{|\cdot|_2 \leq 1}(u) - \iota_\Delta^*(-A_{\perp} u)
    \\
    &= \max\cbr{ - \max_i (-A_{\perp}u)_i : |u|_2 \leq 1 }
    \\
    &= -\min\cbr{ \max_i (A_{\perp}u)_i : |u|_2 \leq 1 },
  \end{align*}
  and moreover every primal-dual optimal pair $(\baru,\barq)$ satisfies
  $A_{\perp}^{\top}\bar{q}\in\partial\left(\iota_{|\cdot|_2 \leq 1}\right)(-\bar{u})$,
  which means $\bar{u}=-A_{\perp}^{\top}\bar{q}/\gamma$.

  It only remains to show that $\baru$ is unique.
  Since $\gamma > 0$,
  necessarily any primal optimum has unit length,
  since the objective value will only decrease by increasing the length.
  Consequently, suppose $u_1$ and $u_2$ are two primal optimal unit vectors.
  Then $u_3 := (u_1 + u_2)/2$ would satisfy
  \[
    \max_i (A_{\perp}u_3)_i
    = \frac 1 2 \max_i (A_{\perp}u_1 + A_{\perp}u_2)_i
    \leq \frac 1 2 \del{ \max_i (A_{\perp}u_1) + \max_j(A_{\perp}u_2)_j }
    = \max_i (A_{\perp}u_1)_i,
  \]
  but then the unit vector $u_4 := u_3 / |u_3|$ would have $|u_4| > |u_3|$ when $u_1\neq u_2$,
  which implies $\max_i(A_{\perp}u_4) < \max_i (A_{\perp}u_3) = \max_i (A_{\perp}u_1)$, a contradiction.
\end{proof}

The proof of \Cref{fact:struct} follows.

\begin{proofof}{\Cref{fact:struct}}
  Partition the rows of $A$ into $A_c$ and $A_S$ as follows.
  For each row $i$, put it in $A_c$ if there exists $u_i$ so that
  $Au_i \leq 0$ (coordinate-wise) and $(Au_i)_i < 0$;
  otherwise, when no such $u_i$ exists, add this row to $A_S$.
  Define $S:= \SPAN(A_S^\top)$, the linear span of the rows of $A_S$.
  This has the following consequences.
  \begin{itemize}
    \item
      To start, $S^\perp = \SPAN(A_S^\top)^\perp = \ker(A_S) \subseteq \ker(A)$.
    \item
      For each row $i$ of $A_c$,
      the corresponding $u_i$ has $A_Su_i = 0$, since otherwise $Au_i\leq 0$ implies there would be a negative coordinate of $A_Su_i$,
      and this row should be in $A_c$ not $A_S$.
      Combining this with the preceding point, $u_i \in \ker(A_S) = S^\perp$.
      Define $\tilde u := \sum_i u_i \in S^\perp$, whereby $A_c \tilde u < 0$ and $A_S \tilde u = 0$.
      Lastly, $\tilde u \in \ker(A_S)$ implies moreover that $A_\perp \tilde u = A_c \tilde u < 0$.
      As such, when $A_c$ has a positive number of rows,
      \Cref{fact:margin_duality} can be applied,
      resulting in the desired unique $\baru = -A_\perp^\top/\gamma \in S^\perp$ with $\gamma > 0$.

    \item
      $S$, $S^\perp$, $A_S$, and $A_c$, and $\baru$ are unique and constructed from $A$ alone,
      with no dependence on $\ell$.
    \item
      If $A_c$ is empty, there is nothing to show, thus suppose $A_c$ is nonempty.
      Since $\lim_{z\to-\infty} \ell(z) = 0$,
      \[
        0
        \leq \inf_{w\in \R^d} L(A_cw)
        \leq \inf_{w\in S^\perp} L(A_cw)
        \leq \inf_{u \in S^\perp} L(A_c u)
        \leq \lim_{r \to \infty} L(r\cdot A_c \baru)
        = 0.
      \]
      Since these inequalities start and end with 0, they are equalities, and consequently
      $\inf_{w\in \R^d} = \inf_{u\in S^\perp} L(A_cu) = 0$.
      Moreover,
      \begin{align*}
        \inf_{w\in\R^d} L(A w)
        &\leq
        \inf_{\substack{v\in S \\ u \in S^\perp}}
        \del{ L(A_S (u+ v) + L(A_c (u+v)) }
        =
        \inf_{v \in S}
        \del{ L(A_S v) + \inf_{u\in S^\perp} L(A_c (u+v)) }
        \\
        &\leq
        \del{ \inf_{v \in S}
        L(A_S v)}
        + \del{ \inf_{u\in S^\perp} L(A_c u) }
        =
        \del{ \inf_{v \in S}
        L(A_S v)}
        \leq
        \inf_{w \in \R^d}
        L(A w).
      \end{align*}
      which again is in fact a chain of equalities.

    \item
      For every $v \in S$ with $|v|> 0$, there exists a row $a$ of $A_S$ such that $\ip{a}{v} > 0$.
      To see this, suppose contradictorily that $A_S v \leq 0$.
      It cannot hold that $A_S v = 0$, since $v\neq 0$ and $\ker(A_S) \subseteq S^\perp$.
      this means $A_Sv \leq 0$ and moreover $(A_Sv)_i < 0$ for some $i$.
      But since $A\baru \leq 0$ and $A_c\baru < 0$,
      then for a sufficiently large $r > 0$, $A (v+r\baru) \leq 0$ and $(A_S(v+r\baru))_j< 0$,
      which means row $j$ of $A_S$ should have been in $A_c$, a contradiction.

    \item
      Consider any $v\in S\setminus \{0\}$.
      By the preceding point, there exists a row $a$ of $A_S$ such that $\ip{a}{v} > 0$.
      Since $\ell(0) > 0$ (because $\ell'' > 0)$ and $\lim_{z\to-\infty}$)
      and $\lim_{z\to-\infty} = 0$, there exists $r > 0$ so that
      $\ell(-r\ip{a}{v}) = \ell(0)/2$.  By convexity, for any $t>0$,
      setting $\alpha := r / (t+r)$ and noting $\alpha\ip{a}{tv} + (1-\alpha)\ip{a}{-rv} = 0$,
      \[
        \alpha \ell(t\ip{a}{v})
        \geq
        \ell(0)
        -
        (1-\alpha)\ell(-r\ip{a}{v})
        = \del{\frac{1+\alpha}{2}}\ell(0),
      \]
      thus $\ell(t\ip{a}{v}) \geq \del{\frac{1+\alpha}{2\alpha}}\ell(0)
      = \del{\frac{t+2r}{2r}}\ell(0)$, and
      \[
        \lim_{t\to\infty} \frac {L(tAv) - L(0)}{t}
        \geq
        \lim_{t\to\infty} \frac {\ell(t\ip{a}{v}) - n\ell(0)}{t}
        \geq
        \lim_{t\to\infty} \frac{\ell(0)}{2r}\del{\frac{(t+2r) - 2nr}{t}}
        \geq
        \frac{\ell(0)}{2r}
        >
        0.
      \]
      Consequently, $L\circ A$ has compact sublevel sets over $S$
      \citep[Proposition B.3.2.4]{HULL}.

    \item
      Note $\nabla^2 L(v) = \diag(\ell''(v_1),\ldots,\ell''(v_n))$.
      Moreover, since $\ker(A) \subseteq S^\perp$,
      then the image $B_0 := \{ Av : v\in S, |v| = 1 \}$ over the surface
      of the ball in $S$ through $A$ is a collection of vectors with positive length.
      Thus for any compact subset $S_0 \subseteq S$,
      \begin{align*}
        \inf_{\substack{v_1\in S_0\\v_2\in S, |v_2|=1}}
        v_2^\top \nabla^2(L\circ A)(v_1) v_2
        &=
        \inf_{\substack{v_1\in S_0\\v_2\in S, |v_2|=1}}
        (Av_2)^\top \nabla^2 L(Av_1) (Av_2)
        =
        \inf_{\substack{v_1\in S_0\\v_3 \in B_0}}
        v_3^\top \nabla^2 L(Av_1) v_3
        \\
        &\geq
        \inf_{\substack{v_1\in S_0\\v_3 \in B_0}}
        |v_3|^2 \min_i \ell''((v_1)_i)
        > 0,
      \end{align*}
      the final inequality since the minimization is of a continuous function over
      a compact set, thus attained at some point, and the infimand is positive over the domain.
      Consequently, $L\circ A$ is strongly convex over compact subsets of $S$.

    \item
      Since $L\circ A$ is strongly convex over $S$ and moreover has bounded sublevel sets over $S$,
      it attains a unique optimum over $S$.

  \end{itemize}
\end{proofof}

\section{Omitted proofs from \Cref{sec:risk}}

To start, note how the three key lemmas provided in the main text
lead to a proof of \Cref{thm:risk_converge}.

\begin{proofof}{\Cref{thm:risk_converge}}
  Since $\eta_j \leq 1$,
  then \Cref{fact:smooth_ineq}
  guarantees both that the desired smoothness inequality
  holds and that function values decrease,
  whereby $\heta_j \leq \eta_j \cR(w_j) \leq \eta_j$.
  Thus, by \Cref{lem:magic_ineq},
  for any $z\in\R^d$,
  \begin{align*}
    2 \del{\sum_{j<t}\eta_j} \del{\cR(w_t) - \cR(z)}
    &\leq
    2 \sum_{j<t} \eta_j \del{\cR(w_j) - \cR(z)}
    + 2 \sum_{j<t} \eta_j \del{\cR(w_{j+1} - \cR(w_j)}
    \\
    &\leq
    2\sum_{j<t}\eta_j \del{ \cR(w_j) - \cR(z) }
    - \sum_{j<t} \frac {\eta_{j}}{1 - \nicefrac {\eta_{j}}{2} }
    \del{ \cR(w_j) - \cR(w_{j+1}) }
    \\
    &\leq
    \envert{w_{0} - z}^2 - \envert{ w_t - z }^2
    \\
    &\leq
    \envert{z}^2.
  \end{align*}
  Consequently, by the choice
  $z:=\barv + \baru(\nicefrac{\ln(t)}{\gamma}$
  and \Cref{lem:fixedDirRate},
  \[
    \cR(w_t)
    \leq
    \cR(z) + \frac{|z|^2}{2\sum_{j<t} \eta_j}
    \leq
    \inf_w \cR(w) + \frac {\exp(|\barv|)}{t} + \frac{|\barv|^2 + \ln(t)^2/\gamma^2}{2\sum_{j<t} \eta_j}.
  \]
\end{proofof}

To fill out the proof, first comes the smoothness-based risk guarantee.

\begin{proofof}{\Cref{lem:magic_ineq}}
  Define $r_j := \eta_j(1 - \beta\eta_j/2)$.
  For any $j$,
  \begin{align*}
    \envert{w_{j+1} - z}^2
    &=
    \envert{w_{j} - z}^2 + 2\eta_{j}\ip{\nf(w_j)}{z - w_j} + \eta_{j}^2 \envert{ \nf(w_j) }^2
    \\
    &\leq
    \envert{w_{j} - z}^2 + 2\eta_{j}\del{ f(z) - f(w_j) } + \frac {\eta_{j}^2}{r_{j}} \del{ f(w_j) - f(w_{j+1}) }.
  \end{align*}
  Summing this inequality over $i\in \cbr{0, \ldots, t-1}$ and rearranging gives the bound.
\end{proofof}

With that out of the way, the remainder of this subsection
establishes smoothness properties of $\cR$.
For convenience, for the rest of this subsection define
$w' := w - \eta \nR(w) =  w - \eta \AT\nL(Aw) / n$.
Additionally, suppose throughout that $\ell$ is twice differentiable
and $\max_i |A_i|\leq 1$.

\begin{lemma}
  \label[lemma]{fact:smooth:1}
  For any $w\in \R^d$,
  \begin{align*}
    \cR(w')
    &\leq \cR(w) - \eta |\nR(w)|^2
    + \frac{\eta^2} 2 |\nR(w)|^2 \max_{v\in [w,w']} \sum_i \ell''(A_iv)/n.
  \end{align*}
\end{lemma}
\begin{proof}
  By Taylor expansion,
  \begin{align*}
    \cR(w')
    &\leq \cR(w) - \eta |\nR(w)|^2
    +
    \frac 1 2 \max_{v\in [w,w']} \sum_i (A_i(w-w'))^2 \ell'' (A_iv) /n.
  \end{align*}
  By H\"older's inequality,
  \begin{align*}
    \max_{v\in [w,w']}\sum_i (A_i(w-w'))^2 \ell'' (A_iv)
    &\leq \max_{v\in [w,w']} |A(w-w')|_\infty^2 \sum_i \ell'' (A_iv).
  \end{align*}
  Since $\max_i |A_i| \leq 1$,
  \[
    |A (w - w') |_\infty^2
    = \eta^2 | A \nR(w) |_\infty^2
    = \eta^2 \max_i \ip{A_{i}}{\nR(w)}^2
    \leq \eta^2 \max_i |A_{i}|^2 |\nR(w)|^2
    \leq \eta^2 |\nR(w)|^2.
  \]
  Thus
  \begin{align*}
    \cR(w')
    &\leq \cR(w) - \eta |\nR(w)|^2
    + \frac{\eta^2} 2 |\nR(w)|^2 \max_{v\in [w,w']} \sum_i \ell''(A_iv)/n.
  \end{align*}
\end{proof}

\begin{lemma}
  \label[lemma]{fact:smooth:2}
  Suppose $\ell',\ell'' \leq \ell$ and $\ell$ is convex.
  Then, for any $w\in\R^d$,
  \[
    \max_{v\in [w,w']} \sum_i \ell''(A_iv)/n
    \leq \max\cbr{ \cR(w), \cR(w') }.
  \]
  Define $\heta := \eta\cR(w)$ and suppose $\heta \leq 1$;
  then $\cR(w') \leq \cR(w)$
  and
  \[
    \cR(w')\leq \cR(w)\del{ 1 - \heta (1-\heta/2) \frac { |\nR(w) |^2 }{ \cR(w)^2 } }.
  \]
\end{lemma}
\begin{proof}
  Since $\ell'' \leq \ell$ and $\ell$ is convex,
  \[
    \max_{v\in [w,w']} \sum_i \ell'' (A_iv)/n
    \leq \max_{v\in [w,w']} \sum_i \ell(A_iv)/n
    = \max_{v\in [w,w']} \cR(v)
    = \max\cbr{ \cR(w), \cR(w') }.
  \]
  Combining this, the choice of $\eta$,
  and \Cref{fact:smooth:1},
  \begin{align*}
    \cR(w')
    &\leq \cR(w) - \eta |\nR(w)|^2
    + \frac{\eta^2} 2 |\nR(w)|^2 \max\cbr{ \cR(w), \cR(w') } \\
    & =\cR(w)-\frac {\heta |\nR(w)|^2}{\cR(w)}
    \del{1 - \frac{\heta} 2 \frac {\max\cbr{\cR(w), \cR(w')}}{\cR(w)}}.
  \end{align*}

  As a final simplification,
  suppose $\cR(w') > \cR(w)$;
  since $\heta\leq 1$ and $\ell' \leq \ell$ and $\max_i |A_i|\leq 1$,
  \[
    \frac{\cR(w')}{\cR(w)}-1
    \le \frac {\heta |\nR(w)|^2}{\cR(w)^2}
    \del{\frac{\heta} 2 \frac {\cR(w')}{\cR(w)}-1}
    \le\heta\del{\frac{\heta} 2 \frac {\cR(w')}{\cR(w)}-1}
    \le \frac{\heta} 2 \frac {\cR(w')}{\cR(w)}-1
    \le \frac{1} 2 \frac {\cR(w')}{\cR(w)}-1,
  \]
  a contradiction.
  Therefore $\cR(w') \leq \cR(w)$, which in turn implies
  \begin{align*}
    \cR(w')
    &\leq \cR(w) - \frac {\heta |\nR(w)|^2}{\cR(w)}
    \del{1 - \frac{\heta} 2 }.
  \end{align*}
\end{proof}

Together, these pieces prove the desired smoothness inequality.

\begin{proofof}{\Cref{fact:smooth_ineq}}
  For any $j < t$,
  by \Cref{fact:smooth:2} and the definition of $\gamma_j$,
  \[
    \cR(w_{j+1})\leq \cR(w_j)\del{ 1 - \heta_j (1-\heta_j/2) \frac { |\nR(w_j) |^2 }{ \cR(w_j)^2 } }
    =
    \cR(w_j)\del{ 1 - \heta_j (1-\heta_j/2) \gamma_j^2 }.
  \]
  Applying this recursively gives the bound.

  Lastly,
  \[
    |w_t|
    = \envert{ \sum_{j < t} \heta_j q_j }
    \leq \sum_{j < t} \envert{ \heta_j q_j }
    = \sum_{j < t} \heta_j \gamma_j.
  \]
\end{proofof}

\section{Omitted proofs from \Cref{sec:param}}

This section will be split into subsections paralleling
those in \Cref{sec:param}.

\subsection{Bounding $|\Pi_\perp w_t|$}

To start, the proof of the smoothness-based convergence guarantee,
but with sensitivity to $A_c$.

\begin{proofof}{\Cref{lem:magic_ineq_perp}}
    Fix any $u\in S^{\perp}$.
    Expanding the square,
    \begin{align*}
      \envert{ \Pip w_{j+1} - u }^2
      &=
      \envert{ \Pip w_j - u }^2
      +2\eta_j\ip{\Pip\nR(w_j)}{u - \Pip w_j}
      + \eta_j^2 \envert{ \Pip \nR(w_j) }^2,
    \end{align*}
    whose two key terms can be bounded as
    \begin{align*}
      \ip{\Pip\nR(w_j)}{u - \Pip w_j} & =\ip{\nR_c(w_j)}{u - \Pip w_j}
      \\
       & =\ip{\nR_c(w_j)}{u - w_j}+\ip{\nR_c(w_j)}{w_j - \Pip w_j}
       \\
       & \le\cR_c(u) - \cR_c(w_j)+\ip{\nR_c(w_j)}{w_j - \Pip w_j},
       \\
      \eta_j^2 \envert{ \Pip \nR(w_j) }^2 & \le\eta_j\envert{ \nR(w_j) }^2
      \\
       & \le 2\left(\cR(w_j)-\cR(w_{j+1})\right),
    \end{align*}
    the last inequality making use of smoothness, namely
    \Cref{fact:smooth_ineq}.
    Therefore
    \[
    \envert{ \Pip w_{j+1} - u }^2\le\envert{ \Pip w_j - u }^2+2\eta_j\left(\cR_c(u) - \cR_c(w_j)+\ip{\nR_c(w_j)}{w_j - \Pip w_j}\right)+2\left(\cR(w_j)-\cR(w_{j+1})\right).
    \]
    Applying $\sum_{j<t}$ to both sides and canceling terms yields
    \[
      \envert{ \Pip w_t - u}^2 \le |u|^2+2+\sum_{j<t}^{}2\eta_j\left(\cR_c(u) - \cR_c(w_j)\right)+\sum_{j<t}^{}2\eta_j\ip{\nR_c(w_j)}{w_j - \Pip w_j}
    \]
    as desired.
\end{proofof}

Proving \Cref{fact:wt_lnt} is now split into upper and lower bounds.

\begin{proofof}{upper bound in \Cref{fact:wt_lnt}}
    For a fixed $t\ge1$, define
    \[
      u := \frac {\ln(t)}{\gamma}\baru,
      \qquad
      \ell'_{<t} := \sum_{j<t} \eta_j \frac{|\nL(A_cw_j)|_1}{n},
      \qquad
      R := \sup_{j < t} \envert{ \Pip w_j - w_j }\le |\bar{v}|+\sqrt{\frac{2}{\lambda}},
    \]
    where the last inequality comes from \Cref{fact:param_converge_S}.

    The strategy of the proof is to rewrite various quantities in \Cref{lem:magic_ineq_perp}
    with $\ell'_{<t}$, which after applying \Cref{lem:magic_ineq_perp}
    cancel nicely to obtain an upper bound on $\ell'_{<t}$.
    This in turn completes the proof, since
    \[
      |\Pip w_t|
      \leq \sum_{j<t} \eta_j |\Pip \nabla \cR(w_t)|
      \leq \sum_{j<t} \eta_j |\nabla \cR_c(w_j)|
      \leq \sum_{j<t} \eta_j |\nabla L(A_c w_j)|/n
      = \ell'_{<t}.
    \]

    Proceeding with this plan, first note (similarly to the main text)
    \begin{align*}
    \envert{ \Pip w_t - u }
    &\geq \ip{\Pip w_t - u}{\baru}
    \\
    &= \ip{- \sum_{j<t} \eta_j \Pip \nR(w_j)/n}{\baru} - \ip{u}{\baru}
    \\
    &= \sum_{j<t} \eta_j \ip{-A_\perp^\top \nL(A_cw_j)/n}{\baru} - \frac {\ln(t)}{\gamma}
    \\
    &= \sum_{j<t} \eta_j \frac{\envert{\nL(A_cw_j)}_1}{n} \ip{-A_\perp^\top \frac{\nL(A_cw_j)}{|\nL(A_cw_j)|_1}}{\baru} - \frac {\ln(t)}{\gamma}
    \\
    &\geq \sum_{j<t} \eta_j \frac{\envert{\nL(A_cw_j)}_1}{n} \gamma - \frac {\ln(t)}{\gamma}
    \\
    &= \gamma \ell'_{<t} - \frac {\ln(t)}{\gamma},
    \end{align*}
    and since $\ell' \leq \ell$
    \begin{align*}
    \sum_{j<t} \eta_j \cR_c(w_j)
    &= \sum_{j<t} \eta_j L(A_cw_j)/n
    \\
    &\geq \sum_{j<t} \eta_j |\nL(A_cw_j)|_1/n
    \\
    &= \ell'_{<t},
    \end{align*}
    and
    \begin{align*}
    \sum_{j<t} \eta_j \ip{\nR_c(w_j)}{w_j - \Pip w_j}
    &\leq
    \sum_{j<t} \eta_j \envert{\nR_c(w_j)}\envert{w_j - \Pip w_j}
    \\
    &\leq
    \sum_{j<t} \eta_j \frac{|\nL(A_cw_j)|_1 }{n} \envert{A_c^\top \frac{\nL(A_cw_j)}{|\nL(A_cw_j)|_1}} R
    \\
    &\leq
    R \ell'_{<t}.
    \end{align*}
    Combining these terms with \Cref{lem:magic_ineq_perp},
    \begin{align*}
        2\ell'_{<t}
        + \del{\gamma \ell'_{<t} - \nicefrac{\ln(t)}{\gamma}}^2
        &\leq
        \sum_{j<t} 2\eta_j \cR_c(w_j)
        + \envert{ \Pip w_t - u }^2
        \notag\\
        &\leq
        \envert{u}^2
        +
        \sum_{j<t} 2\eta_j \ip{\nR_c(w_j)}{w_j - \Pip w_j}
        +
        \sum_{j<t}2\eta_j\cR_c(u)
        +
        2
        \notag\\
        &\leq
        \frac {\ln(t)^2}{\gamma^2}
        + 2R \ell'_{<t}
        + \frac 2 t \sum_{j<t} \eta_j
        + 2.
    \end{align*}
    Equivalently,
    \[
    2\ell'_{<t}
    + \gamma^2 (\ell'_{<t})^2
    \leq
    2\ell'_{<t} \ln(t)
    + 2R \ell'_{<t}
    + \frac 2 t \sum_{j<t} \eta_j
    + 2,
    \]
    which implies
    \[
    \ell'_{<t}
    \leq \max\cbr{\frac {4\ln(t)}{\gamma^2}, \frac {4R}{\gamma^2}, \frac 2 {t}\sum_{j<t}\eta_j, 2},
    \]
    since otherwise
    \begin{align*}
        2\ell'_{<t} \ln(t)
        + 2R \ell'_{<t}
        + \frac 2 t \sum_{j<t} \eta_j
        + 2
        &<
        \frac{\gamma^2(\ell'_{<t})^2}{2}
        +
        \frac{\gamma^2(\ell'_{<t})^2}{2}
        +
        \ell'_{<t}
        +
        \ell'_{<t}
        \\
        &\leq
        2\ell'_{<t} \ln(t)
        + 2R \ell'_{<t}
        + \frac 2 t \sum_{j<t} \eta_j
        + 2,
    \end{align*}
    a contradiction.
\end{proofof}

\begin{proofof}{lower bound in \Cref{fact:wt_lnt}}
    First note
    \begin{align*}
        n_c \exp(-|\Pip w_t|)
        & \leq n_c \ell(-|\Pip w_t|)/\ln 2
        \\
        & \leq L(A_c\Pip w_t)/\ln 2
        \\
        & = L\del{A_cw_t - A_c(w_t - \Pip w_t)}/\ln 2
        \\
        & = L\del{A_cw_t - A_c\Pi_S w_t}/\ln 2.
        \\
        & \leq  L\del{A_cw_t +R}/\ln 2.
    \end{align*}
    If $\ell = \lexp$, then $L(A_cw_t + R) = e^R L(A_cw_t)$.
    Otherwise, by Bernoulli's inequality,
    \[
      \sum_i \llog((A_cw_t)_i + R)
      = \sum_i \ln\del{ 1 + e^R \exp((A_cw_t)_i) }
      \leq \sum_i e^R \ln\del{ 1 + \exp((A_cw_t)_i) }.
    \]
    Combining these steps, and invoking \Cref{fact:struct},
    \begin{align*}
      n_c \ln(2) \exp(-|\Pip w_t|)
      &\leq \exp(R) L(A_c w_t)
      \\
      &= \exp(R)\del{ L(A w_t) - L(A_S w_t) }
      \leq \exp(R)\del{ L(A w_t) - \inf_w L(A w) }.
    \end{align*}
    By \Cref{thm:risk_converge},
    \begin{align*}
      \ln\del{ \cR(w_t)-\inf_w \cR(w) }
      &\le\ln\ \max\cbr{ \frac{2\exp\left(|\barv|\right)}{t},\ {} \frac{|\barv|^2 + \ln(t)^2/\gamma^2}{\sum_{j=0}^{t-1}\eta_j} }
      \\
      &=\max\cbr{ \ln(2) + |\barv| - \ln(t),\ {} \ln\del{|\barv|^2 + \ln(t)^2/\gamma^2} - \ln \del{\sum_{j=0}^{t-1}\eta_j} }.
    \end{align*}
    Together,
    \begin{align*}
      |\Pip w_t|
      &\geq - \ln\del{ \cR(A w_t) - \inf_w \cR(A w) }
      -R + \ln \ln(2) - \ln(n/n_c)
      \\
      &\geq
      \min\cbr{ -\ln(2) - |\barv| + \ln(t),\ {} -\ln\del{|\barv|^2 + \ln(t)^2/\gamma^2} + \ln \del{\sum_{j=0}^{t-1}\eta_j} }
      -R + \ln \ln(2) - \ln(n/n_c).
    \end{align*}
\end{proofof}

\subsection{Parameter convergence when separable}

First, the technical inequality on $\llog$ and $\lexp$.

\begin{proofof}{\Cref{fact:log_approx}}
  The claims are immediate for $\ell=\lexp$, thus consider $\ell=\llog$.
  First note that $r \mapsto (e^r - 1) / r$ is increasing and not smaller than 1
  when $r\geq 0$.
  Now set $r := \llog(z)$, whereby $\llog'(z)= e^z / (1+e^z) = (e^r-1)/e^r$.
  Suppose $r\le\epsilon$; since $\exp(\cdot)$ lies above its tangents,
  then $1-\eps \leq 1-r \leq e^{-r}$,
  and
  \[
    \frac{\llog'(z)}{\llog(z)}=\frac{e^r-1}{re^r}\ge \frac{1}{e^r}\ge1-\epsilon.
  \]

  For $\lexp(z) \leq 2\llog(z)$, note
  \[
    \frac{\lexp(z)}{\llog(z)}=\frac{e^r-1}{r}
  \]
  is increasing for $r = \llog(z) > 0$, and $e-1<2$.
\end{proofof}

Leveraging this inequality, the following proof handles \Cref{fact:min_norm} in the general
separable case (not just $\lexp$, as in the main text).

\begin{proofof}{\Cref{fact:min_norm} when $A = A_c = A_\perp$ (separable case)}
    Let $0<\epsilon\le1$ be arbitrary,
    and select $t_0$ so that $\cR(w_{t_0}) \leq \eps/n$. By \Cref{fact:smooth_ineq}, since $\eta_j\le1$, the loss decreases at each step, and thus for any $t\ge t_0$, $\cR(w_t)\le\eps/n$. Now let $t\ge t_0$ and $\cR(w)\le\cR(w_t)$ with arbitrary $t$ and $w$.
    By \Cref{fact:ip} and \Cref{fact:smooth_ineq},
    \begin{equation}\label{eq:sep_ineq1}
        \begin{split}
            \frac{1}{2}\left|\frac{w}{|w|}-\bar{u}\right|^2 & \le1+\frac{\ln\cR(w_t)}{|w|\gamma}+\frac{g^*(\bar{q})}{|w|\gamma}+\frac{\ln2}{|w|\gamma} \\
             & \le1+\frac{\ln\cR(w_{t_0})}{|w|\gamma}-\frac{\sum_{j=t_0}^{t-1}\hat{\eta}_{j}(1-\hat{\eta}_{j}/2)\gamma_{j}^2}{|w|\gamma}+\frac{g^*(\bar{q})}{|w|\gamma}+\frac{\ln2}{|w|\gamma} \\
             & \le1+\frac{\ln(\epsilon/n)}{|w|\gamma}-\frac{\sum_{j=t_0}^{t-1}\hat{\eta}_{j}\gamma_{j}^2}{|w|\gamma}+\frac{\sum_{j=t_0}^{t-1}\hat{\eta}_{j}^2\gamma_{j}^2/2}{|w|\gamma}+\frac{\ln n}{|w|\gamma}+\frac{\ln2}{|w|\gamma} \\
             & \le1-\frac{\sum_{j=t_0}^{t-1}\hat{\eta}_{j}\gamma_{j}^2}{|w|\gamma}+\frac{1+\ln2}{|w|\gamma},
        \end{split}
    \end{equation}
    where (similarly to the proof for $\lexp$ in the main text)
    the last inequality uses the following smoothness consequence (cf. \Cref{fact:smooth_ineq} with $\eta_j\leq 1$)
    \begin{equation*}
        \sum_{j=t_0}^{t-1}\hat{\eta}_{j}^2\gamma_{j}^2=\sum_{j=t_0}^{t-1}\eta_{j}^2|\nR(w_j)|^2\le2\sum_{j=t_0}^{t-1}(\cR(w_{j})-\cR(w_{j+1}))\le2.
    \end{equation*}
    For $j\ge t_0$, by \Cref{fact:log_approx} and
    $\gamma = \min\cbr{ |A_\perp^\top q | : q \geq 0, \sum_i q = 1 }$ (cf. \Cref{fact:struct}),
    \begin{equation}\label{eq:sep_ineq2}
      \gamma_j=|\nabla(\ln\cR)(w_j)|=\frac{|\AT\nL(w_j)|}{L(Aw_j)}=\frac{|\AT\nL(Aw_j)|}{|\nL(Aw_j)|_1}\frac{|\nL(Aw_j)|_1}{L(Aw_j)}\ge(1-\epsilon)\gamma,
    \end{equation}
    Invoking \cref{eq:sep_ineq1} and \cref{eq:sep_ineq2} with $w=w_t$ or $\barw_t$ (notice that in both cases $|w|=|w_t|$) gives
    \begin{equation}\label{eq:sep_tmp}
        \begin{split}
        \frac{1}{2}\left|\frac{w}{|w_t|}-\bar{u}\right|^2 & \le 1-\frac{\sum_{j=t_0}^{t-1}\hat{\eta}_{j}\gamma_{j}^2}{|w_t|\gamma}+\frac{1+\ln2}{|w_t|\gamma} \\
         & \le 1-\frac{\sum_{j=t_0}^{t-1}\hat{\eta}_{j}\gamma_{j}}{|w_t|}(1-\epsilon)+\frac{1+\ln2}{|w_t|\gamma} \\
         & = 1-\frac{|w_{t_0}|+\sum_{j=t_0}^{t-1}\hat{\eta}_{j}\gamma_{j}}{|w_t|}(1-\epsilon)+\frac{|w_{t_0}|}{|w_t|}(1-\epsilon)+\frac{1+\ln2}{|w_t|\gamma} \\
         & \le\epsilon+\frac{|w_{t_0}|}{|w_t|}+\frac{1+\ln2}{|w_t|\gamma}.
     \end{split}
    \end{equation}

    Next, $\epsilon$ and $t_0$ are tuned as follows.  First, set $\epsilon:= \min\{4/|w_t|,1\}$;
    this is possible if the corresponding $t_0\le t$, or equivalently $\cR(w_t)\le4/n|w_t|$.
    This in turn is true as long as $t\ge5$ and
    \[
      \frac{t}{(\ln t)^3}\ge \frac{n}{\gamma^4},
    \]
    because then $(\ln t)^2\ge2$, and since $\gamma\le1$, $|w_t|\le4\ln t/\gamma^2$ given by \Cref{fact:wt_lnt}, \Cref{thm:risk_converge} gives
    \[
      \cR(w_t)\le \frac{1}{t}+\frac{(\ln t)^2}{2\gamma^2t}\le \frac{(\ln t)^2}{2\gamma^2t}+\frac{(\ln t)^2}{2\gamma^2t}=\frac{(\ln t)^2}{\gamma^2t}\le \frac{\gamma^2}{n\ln t}\le \frac{4}{n|w_t|}.
    \]

    By \Cref{thm:risk_converge}, $\cR(w_{t_0})\le\epsilon/n$ if
    \[
      \frac{1}{t_0}+\frac{(\ln t_0)^2}{\gamma^2t_0}\le \frac{\epsilon}{n}.
    \]
    By \Cref{fact:wt_lnt}, $|w_{t_0}|\le \cO\left(\ln(n/\epsilon)\right)/\gamma^2$. Together with \cref{eq:sep_tmp},
    \begin{equation*}
        \max\left\{\left|\frac{w_t}{|w_t|}-\bar{u}\right|^2,\left|\frac{\barw_t}{|\barw_t|}-\bar{u}\right|^2\right\}\le \frac{\cO\left(\ln n+\ln\left(|w_t|\right)\right)}{|w_t|\gamma^2},
    \end{equation*}
    where the constants hidden in the $\cO$ do not depend on the problem.
    Combining this with the lower and upper bounds on $|w_t|$ in \Cref{fact:wt_lnt},
    \begin{equation*}
        \max\left\{\left|\frac{w_t}{|w_t|}-\bar{u}\right|^2,\left|\frac{\barw_t}{|\barw_t|}-\bar{u}\right|^2\right\}\le\cO\left(\frac{\ln n+\ln\ln t}{\gamma^2\ln t}\right).
    \end{equation*}
\end{proofof}

\subsection{Parameter convergence in general}

For convenience in these proofs, define $v_t := \Pi_S w_t$.

The general application of Fenchel-Young is as follows.

\begin{proofof}{\Cref{fact:gen:ip}}
  First note that
  \[
    A_\perp w
    =
    A_c \Pip w
    =
    A_c w + A_c(\Pip w - w)
    =
    A_c w - A_c \Pi_Sw,
  \]
  thus
  \begin{align*}
    \ip{\barq}{A_\perp w}
    &=
    \ip{\barq}{A_c w}
    - \ip{\barq}{A_c\Pi_S(w)}
    \leq
    \ip{\barq}{A_c w}
    + |\barq|_1 |A_c\Pi_S(w)|_\infty\\
    &=
    \ip{\barq}{A_c w}
    + \max_i  (A_c)_{i:} \Pi_S(w)
    \leq
    \ip{\barq}{A_c w_t}
    + |\Pi_S(w)|.
  \end{align*}
  Thus
  \begin{align*}
    \frac{\ip{\baru}{w}}{|\baru|\cdot|w|}
    &=
    \frac{-\ip{A_{\perp}^\top \barq}{w}}{\gamma|w|}
    =
    \frac{-\ip{\barq}{A_{\perp}w}}{\gamma|w|}
    \\
    &=
    \frac{-\ip{\barq}{A_cw}}{\gamma|w|}-\frac{|\Pi_S(w)|}{\gamma|w|}
    \\
    &\geq
    \frac{-\ln\cR_{c,\exp}(w) - g^*(\barq)}{\gamma|w|}-\frac{|\Pi_S(w)|}{\gamma|w|}
    \\
    &\geq
    \frac{-\ln\cR_{c}(w_t) -\ln2 - g^*(\barq)}{\gamma|w|}-\frac{|\Pi_S(w)|}{\gamma|w|}
    \\
    &\geq
    \frac{-\ln \del{ \cR(w_t) -\bar{\cR} } - \ln2 - g^*(\barq)}{\gamma|w|}-\frac{|\Pi_S(w)|}{\gamma|w|}.
  \end{align*}
\end{proofof}

Next, the adjustment of \Cref{fact:smooth_ineq} to upper bounding $\ln(\cR(w_t) - \bar\cR)$,
which leads to an upper bound with $\gamma\gamma_i$ rather than $\gamma_i^2$,
and the necessary cancellation.

\begin{proofof}{\Cref{fact:gen:iter}}
  The first inequality implies the second via the same direct induction in \Cref{fact:smooth_ineq},
  so consider the first inequality.

  Making use of \Cref{fact:smooth:1,fact:smooth:2} and proceeding as in \Cref{fact:smooth:2},
  \begin{align}
    \cR(w_{j+1}) - \bar{\cR}
    &\leq
    \cR(w_j) - \bar{\cR} - \eta_j|\nR(w_j)|^2 + \frac {\eta_j^2\cR(w_j)}{2}|\nR(w_j)|^2
    \notag\\
    &\leq
    \del{ \cR(w_j) - \bar{\cR} }
    \del{1 - \frac{\eta_j |\nR(w_j)|^2}{\cR(w_j) - \bar{\cR}}\del{1 - \eta_j \cR(w_j)/2}}
    \notag\\
    & =
    \del{ \cR(w_j) - \bar{\cR} }
    \del{1 - \frac{|\nR(w_j)|}{\cR(w_j) - \bar{\cR}}\cdot\frac{\eta_j\cR(w_j)|\nR(w_j)|}{\cR(w_j)}\del{1 - \eta_j \cR(w_j)/2}}
    \notag\\
    &\leq
    \del{ \cR(w_j) - \bar{\cR} }
    \del{1 - \frac{|\nR(w_j)|}{\cR(w_j) - \bar{\cR}}\cdot\heta_j\gamma_j\del{1 - \heta /2}}.
    \label{eq:meh}
  \end{align}
  Next it will be shown, by analyzing two cases, that
  \begin{equation}
    \frac{|\nR(w_j)|}{\cR(w_j) - \bar{\cR}} \geq r\gamma. \label{eq:meh:2}
  \end{equation}
  In the following, for notational simplicity let $w$ denote $w_j$.
  \begin{itemize}
    \item
      Suppose $\cR_c(w) < r\del{ \cR(w) - \bar{\cR}}$.
      Consequently,
      \[
        \cR_S(w)-\bar{\cR}>(1-r)\del{ \cR(w) - \bar{\cR}}.
      \]
      Then, since $4\del{ \cR(w) - \bar{\cR}} \leq 2\lambda(1-r)$,
      \begin{align*}
        \frac{|\nR(w)|}{\cR(w) - \bar{\cR}}
        &\geq
        \frac {- |\nR_c(w)| + |\nR_S(w)|}{\cR(w) - \bar{\cR}}
        \\
        &\geq
        \frac {- \cR_c(w) + \sqrt{2\lambda (\cR_S(w) - \bar{\cR}) }}{\cR(w) - \bar{\cR}}
        \\
        &>
        \frac {- r (\cR(w) - \bar{\cR}) + \sqrt{2\lambda (1-r)(\cR(w) - \bar{\cR}) }}{\cR(w) - \bar{\cR}}
        \\
        &\geq
        \frac {(2 - r) (\cR(w) - \bar{\cR})}{\cR(w) - \bar{\cR}}
        \\
        &\geq
        1 \geq r\gamma.
      \end{align*}
    \item
      Otherwise, suppose $\cR_c(w) \geq r\del{ \cR(w) - \bar{\cR}}$.
      Using an expression inspired by a general analysis of AdaBoost
      \citep[Lemma 16 of journal version]{mukherjee_rudin_schapire_adaboost_convergence_rate},
      and introducing $(1-\eps)$ by invoking \Cref{fact:log_approx} as in the separable case,
      \[
        |\nR(w)|\ge \langle -\bar{u},\nR(w)\rangle=\langle -A\bar{u},\nL(Aw)/n\rangle=\langle -A_c\bar{u},\nL(A_cw)/n\rangle\ge\gamma(1-\epsilon)\cR_c(w),
      \]
      Thus
      \[
        \frac{|\nR(w)|}{\cR(w) - \bar{\cR}}
        \geq
        \frac {\gamma(1-\epsilon)\cR_c(w)}{\cR(w) - \bar{\cR}}
        \geq
        \frac {\gamma(1-\epsilon) \del{ r\del{ \cR(w) - \bar{\cR}} }}{\cR(w) - \bar{\cR}}
        = r \gamma(1-\epsilon).
      \]
  \end{itemize}
  Combining \cref{eq:meh} with \cref{eq:meh:2},
  \begin{align*}
      \cR(w_{j+1}) - \bar{\cR}
      &\leq
      \del{ \cR(w_j) - \bar{\cR} }\del{1-r(1-\epsilon)\gamma\gamma_j\heta_j\del{1-\heta_j/2}}.
  \end{align*}
\end{proofof}

Next, the proof of the intermediate inequality by combining
\Cref{fact:gen:iter} and \Cref{fact:gen:ip}.

\begin{proofof}{\Cref{fact:gen:ip:meh}}
  By \Cref{fact:smooth_ineq}, since $\eta_j\le1$, the loss decreases at each step, and thus for any $t\ge t_0$, $\cR(w_t)\le\eps/n$.
  Combining \Cref{fact:gen:ip} and \Cref{fact:gen:iter},
  \begin{align*}
    \frac{\ip{\baru}{w}}{|\baru|\cdot|w|}
    &\geq
    \frac{-\ln \del{ \cR(w) -\bar{\cR} }}{\gamma|w|}-\frac{\ln2 + g^*(\barq)+|\Pi_S(w)|}{\gamma|w|}.
    \\
    &\geq
    \frac{
      r(1-\epsilon)\gamma \sum_{j = {t_0}}^{t-1} \heta_j (1-\heta_j/2) \gamma_j
      -\ln \del{ \cR(w_{t_0}) - \bar{\cR} }}{\gamma|w|}-\frac{\ln2 + g^*(\barq)+|\Pi_S(w)|}{\gamma|w|}
    \\
    &\geq
    \frac{
      r(1-\epsilon)\sum_{j = {t_0}}^{t-1} \heta_j (1-\heta_j/2) \gamma_j
      }{|w|}-\frac{\ln(\epsilon/n)}{\gamma|w|}-\frac{\ln2+\ln n+|\Pi_S(w)|}{\gamma|w|}
    \\
    &\geq
    \frac{r(1-\epsilon)\sum_{j = t_0}^{t-1} \heta_j (1-\heta_j/2) \gamma_j}{|w|}-\frac{\ln2}{\gamma|w|}-\frac{|\Pi_S(w)|}{\gamma|w|}.
  \end{align*}
\end{proofof}

The pieces are in place to prove parameter convergence in general.

\begin{proofof}{general case in \Cref{fact:min_norm}}
  The guarantee on $\bar v_t$ and the $A_c=\emptyset$ case have been discussed in \Cref{fact:param_converge_S},
  therefore assume $A_c\ne\emptyset$.
  The proof will proceed via invocation of the Fenchel-Young scheme in \Cref{fact:gen:ip:meh},
  applied to $w\in\{w_t,\barw_t\}$ since $\cR(\barw_t)\leq\cR(w_t)$.

  It is necessary to first control the warm start parameter $t_0$.
  Fix an arbitrary $\epsilon\in(0,1)$, set $r:=1-\nicefrac{\epsilon}{3}$,
  and let $t_0$ be large enough such that
  \begin{equation}\label{eq:gen_tmp}
    \cR(w_{t_0})-\bar{\cR}\le\frac{\epsilon}{3n},
    \qquad\cR(w_{t_0})-\bar{\cR}\le \frac{\lambda(1-r)}{2}=\frac{\lambda\epsilon}{6},
    \qquad1-\frac{\heta_{t_0}}{2}\ge1-\frac{\eta_{t_0}}{2}\ge1-\frac{\epsilon}{3}.
  \end{equation}
  By \Cref{thm:risk_converge} and the choice of step sizes, it is enough to require
  \[
    \frac{\exp(|\bar{v}|)}{t_0}\le\min\left\{\frac{\eps}{6n},\ {} \frac{\lambda\eps}{12}\right\},
    \qquad
    \frac{|\barv|^2+\ln(t_0)^2/\gamma^2}{2\gamma^2\sqrt{t_0}}\le\min\left\{\frac{\eps}{6n},\ {} \frac{\lambda\eps}{12}\right\},
    \qquad
    \frac{1}{2\sqrt{t_0+1}}\le \frac{\eps}{3}.
  \]
  Therefore, choosing $t_0=\widetilde{\cO}\left(\frac{n^2}{\epsilon^2}\right)$ suffices.

  Invoking \Cref{fact:gen:ip:meh} with the above choice for $w\in\cbr{w_t,\barw_t}$,
  \begin{equation}\label{eq:gen_ineq_1}
      \begin{split}
          \frac{1}{2}\left|\frac{w}{|w_t|}-\bar{u}\right|^2 & =1-\frac{\ip{\baru}{w}}{|\baru|\cdot|w_t|} \\
           & \le1-\frac{r(1-\nicefrac{\epsilon}{3})\sum_{j = t_0}^{t-1} \heta_j (1-\heta_j/2) \gamma_j}{|w_t|}+\frac{\ln2}{\gamma|w_t|}+\frac{|\Pi_S(w)|}{\gamma|w_t|} \\
           & \le1-\frac{(1-\nicefrac{\epsilon}{3})(1-\nicefrac{\epsilon}{3})\sum_{j = t_0}^{t-1} \heta_j (1-\nicefrac{\epsilon}{3}) \gamma_j}{|w_t|}+\frac{\ln2}{\gamma|w_t|}+\frac{|\Pi_S(w)|}{\gamma|w_T|} \\
           & \le1-\frac{(1-\epsilon)\sum_{j = t_0}^{t-1} \heta_j \gamma_j}{|w_t|}+\frac{\ln2}{\gamma|w_t|}+\frac{|\Pi_S(w)|}{\gamma|w_t|} \\
           & =1-\frac{
              (1-\epsilon) \left(|w_{t_0}|+\sum_{j=t_0}^{t-1} \heta_j  \gamma_j)\right)
              }{|w_t|}+(1-\epsilon)\frac{|w_{t_0}|}{|w_t|}+\frac{\ln2}{\gamma|w_t|}+\frac{|\Pi_S(w)|}{\gamma|w_t|} \\
           & \le\epsilon+\frac{|w_{t_0}|}{|w_t|}+\frac{\ln2}{\gamma|w_t|}+\frac{|\Pi_S(w)|}{\gamma|w_t|}. \\
      \end{split}
  \end{equation}

  Suppose $t\ge5$ and $\sqrt{t}/\ln^3t\ge n(1+R)/\gamma^4$, where $R=\sup_{j<t}|\Pip w_j-w_j|=\cO(1)$ is introduced in \Cref{fact:wt_lnt}.
  As will be shown momentarily,
  $t$ satisfies \eqref{eq:gen_tmp} with
  $
    \epsilon\le \nicefrac{C}{|w_t|}
  $
  for some constant $C$, and therefore this choice of $\epsilon$ can be plugged into \eqref{eq:gen_ineq_1}.
  To see this, note that \Cref{thm:risk_converge} gives
  \begin{align*}
      \cR(w_t)-\bar{\cR} & \le \frac{\exp(|\barv|)}{t}+\frac{|\barv|^2+\ln(t)^2}{2\gamma^2\sqrt{t}} \\
       & \le \frac{\exp(|\barv|)}{\sqrt{t}/\ln (t)^2}+\frac{|\barv|^2}{2\gamma^2\sqrt{t}/\ln (t)^2}+\frac{\ln(t)^2}{2\gamma^2\sqrt{t}} \\
       & \le \frac{\exp(|\barv|)\gamma^4}{n(1+R)\ln t}+\frac{|\barv|^2\gamma^2}{2n(1+R)\ln t}+\frac{\gamma^2}{2n(1+R)\ln t} \\
       & \le C_1 \frac{\gamma^2}{n\cdot4(1+R)\ln t} \\
       & \le C_1 \frac{1}{n|w_t|},
  \end{align*}
  where the last line uses \Cref{fact:wt_lnt}. It can be shown similarly that other parts of \eqref{eq:gen_tmp} hold.

  Continuing with \Cref{eq:gen_ineq_1} but using $\eps\leq \nicefrac{C}{|w_t|}$
  and upper bounding $|w_{t_0}|$ via \Cref{fact:wt_lnt},
  \[
    \left|\frac{w_t}{|w_t|}-\bar{u}\right|^2\le \frac{\cO\left(\ln n+\ln\left(1/|w_t|\right)\right)}{|w_t|\gamma^2}.
  \]
  Lastly, controlling the denominator with the lower bound on $|w_t|$ in \Cref{fact:wt_lnt},
  \[
    \left|\frac{w_t}{|w_t|}-\bar{u}\right|^2\le \cO\left(\frac{\ln n+\ln\ln t}{\gamma^2\ln t}\right).
  \]
\end{proofof}

\end{document}